\documentclass{article}

\PassOptionsToPackage{numbers, compress}{natbib}
\usepackage[final]{nips_2017}

\usepackage[utf8]{inputenc} 
\usepackage[T1]{fontenc}    
\usepackage{hyperref}       
\usepackage{url}            
\usepackage{booktabs}       
\usepackage{amsfonts}       
\usepackage{nicefrac}       
\usepackage{microtype}      

\usepackage{times}
\usepackage{graphicx} 
\usepackage{subfigure} 

\usepackage{algorithm}
\usepackage{algorithmic}

\usepackage{multicol}
\usepackage{mathtools}

\def\supp{Appendix\xspace}

\title{Bridging the Gap Between Value and Policy Based\\ Reinforcement Learning}

\usepackage{bbm}
\usepackage{amsthm}
\usepackage{amsmath}
\usepackage{amssymb}
\usepackage{color}
\usepackage{xspace}
\usepackage{pgfplots, pgfplotstable}

\newcommand{\comment}[1]{}

\renewcommand{\vec}[1]{\boldsymbol{\mathbf{#1}}}
\renewcommand{\mid}[0]{\:\vert\:}

\def\longname{Path Consistency Learning\xspace}
\def\shortname{PCL\xspace}
\def\longuname{Unified Path Consistency Learning\xspace}
\def\shortuname{Unified PCL\xspace}

\def\qop{Q^\circ}

\def\qs{Q^*}
\def\Qstar{Q^*}

\def\expected{\mathbb{E}}
\def\pitheta{\pi_\theta}
\def\Vphi{V_\phi}

\def\Qrho{Q_\rho}
\def\Vrho{V_\rho}
\def\pirho{\pi_\rho}
\def\pistar{\pi^*}
\def\Vstar{V^*}
\def\svstar{v^*}
\def\vop{V^\circ}
\def\piop{\pi^\circ}
\def\svop{v^\circ}
\def\pistar{\pi^*}
\def\vpi{V^\pi}
\def\vstar{V^*}
\def\S{S}

\def\E{E}

\newcommand{\R}[1]{R}
\newcommand{\G}[1]{G}
\newcommand{\HH}[1]{\mathbb{H}}

\newcommand{\Ad}[1]{A}
\newcommand{\C}[3]{C_{{#2},{#3}}}
\newcommand{\CC}[1]{C}

\def\calB{\mathcal{B}}

\def\eg{{\em e.g.,}}
\def\ie{{\em i.e.,}}

\def\vs{{\em vs.}\xspace}

\newtheorem{theorem}{Theorem}
\newtheorem{lemma}[theorem]{Lemma}
\newtheorem{corollary}[theorem]{Corollary}

\newcommand{\figref}[1]{Figure~\ref{#1}}

\newcommand{\theoref}[1]{Theorem~\ref{#1}}

\newcommand{\secref}[1]{Section~\ref{#1}}

\def\temp{\tau}

\def\mro{O_{\text{ER}}}

\def\ento{O_{\text{ENT}}}
\def\pclo{O_{\text{\shortname}}}

\newcommand{\one}[1]{\mathbbm{1}[#1]}

\newcommand{\opensource}{\url{https://github.com/tensorflow/models/tree/master/research/pcl_rl}}

\begin{document} 

\author{
\vspace*{-.2cm}
\begin{tabular}{c@{\hspace*{.8cm}}c@{\hspace*{.8cm}}c@{\hspace*{.8cm}}c}
  {\bf Ofir Nachum$^1$}\addtocounter{footnote}{1}
  & {\bf Mohammad Norouzi} & {\bf Kelvin Xu$^1$} & {\bf Dale Schuurmans}\\[.1cm]
  \multicolumn{4}{c}{\normalfont \texttt{\{ofirnachum,mnorouzi,kelvinxx\}@google.com,~daes@ualberta.ca}}\\[.1cm]
  \multicolumn{4}{c}{\normalfont Google Brain}\\[-.3cm]
\end{tabular}
}

\maketitle

\footnotetext{Work done as a member of the Google Brain Residency program (\url{g.co/brainresidency})}

\begin{abstract}

We establish a new connection between value and policy based 
reinforcement learning (RL) 
based on a relationship between softmax temporal 
value consistency and policy optimality under entropy regularization.
Specifically, we show that softmax consistent action values
correspond to optimal entropy regularized policy probabilities
\emph{along any action sequence}, regardless of provenance.
From this observation, we develop a new RL algorithm,
\emph{\longname (\shortname)}, that minimizes a notion of soft consistency error along 
multi-step action sequences extracted from both on- and off-policy traces.
We examine the behavior of PCL in different 
scenarios
and show that PCL
can be interpreted as generalizing both actor-critic
and Q-learning algorithms.
We subsequently deepen the relationship by showing how a \emph{single} model 
can be used to represent both a policy and the corresponding softmax state values,
eliminating the need for a separate critic.
%
%
The experimental evaluation demonstrates that \shortname
significantly outperforms strong actor-critic and Q-learning
baselines across several benchmarks.\footnote{An implementation of PCL can be found at \opensource}

\end{abstract}

\section{Introduction}
\label{intro}

Model-free RL aims to acquire an effective behavior policy through 
trial and error interaction with a black box environment.
The goal is to optimize the quality of an agent's behavior policy
in terms of the total expected discounted reward.
Model-free RL has a myriad of applications in
games~\cite{atarinature,tesauro1995},
robotics~\cite{kober2013,levine2016end},
and marketing~\cite{li2010,theocharous2015}, to name a few.
Recently, the impact of model-free RL has been expanded through the use of 
deep neural networks,
which promise to replace manual feature engineering 
with end-to-end learning of value and policy representations.
Unfortunately, a key challenge remains how best to 
combine the advantages of value and policy based RL approaches
in the presence of deep function approximators,
while mitigating their shortcomings.
Although recent progress has been made 
in combining value and policy based methods, this issue is not yet settled,
and the intricacies of each perspective are exacerbated by deep models.

The primary advantage of policy based approaches, 
such as REINFORCE \cite{williams92},
is that they directly optimize the quantity of interest while remaining
stable under function approximation (given a sufficiently small learning rate). 
Their biggest drawback is sample inefficiency:
since policy gradients are estimated from rollouts
the variance is often extreme.
Although policy updates can be improved by the use of appropriate geometry
\cite{kakade01,petersetal10,trpo2015},
the need for variance reduction remains paramount.
Actor-critic methods have thus become popular
\cite{schulmaniclr2016,silver14ddpg,sutton1999policy},
because they use value approximators to replace rollout estimates
and reduce variance, at the cost of some bias.
Nevertheless, on-policy learning remains inherently sample inefficient
\cite{guetal17};
by estimating quantities defined by the \emph{current} policy,
either on-policy data must be used,
or updating must be sufficiently slow to avoid significant bias.
Naive importance correction is hardly able
to overcome these shortcomings in practice
\cite{precup2000eligibility,precup2001off}.

By contrast, value based methods, such as Q-learning
\cite{watkins1992q,atarinature,pdqn,wangetal16,mnih2016asynchronous},
can learn from \emph{any} trajectory sampled from the same environment.
Such ``off-policy'' methods are able to exploit data from other sources, such as experts,
making them inherently more sample efficient than on-policy methods~\cite{guetal17}.
Their key drawback is that off-policy learning does not stably interact with
function approximation~\cite[Chap.11]{suttonbook_2nd_ed}.
The practical consequence is that extensive hyperparameter tuning 
can be required to obtain stable behavior.
Despite practical success \cite{atarinature},
there is also little theoretical understanding of how
deep Q-learning might obtain near-optimal objective values.

Ideally, one would like to combine the unbiasedness and stability 
of on-policy training with the data efficiency of off-policy approaches.
This desire has motivated substantial recent work on \emph{off-policy} 
actor-critic 
methods,
where the data efficiency of policy gradient is improved by 
training an off-policy critic
\cite{lillicrap2015continuous,mnih2016asynchronous,guetal17}.
Although such methods have demonstrated improvements over
on-policy actor-critic approaches, they have not resolved 
the theoretical difficulty associated with off-policy learning
under function approximation.
Hence, current methods remain potentially unstable and
require 
specialized algorithmic and theoretical development
as well as delicate tuning to be effective in practice \cite{guetal17, acer, reactor}.

In this paper, we exploit a relationship between
policy optimization under entropy regularization and softmax value consistency
to obtain a new form of stable off-policy learning.
Even though entropy regularized policy optimization is a well studied topic in RL
\cite{williams1991function,todorov2006linearly,todorov10,
ziebart2010modeling,azar,azaretal11,azaretal12,fox}--%
in fact,
one that has been attracting renewed interest from concurrent work
\cite{pgq2017,haarnojaetal17}--%
we contribute new observations to this study that are essential for the 
methods we propose:
first, we identify a strong form of path consistency
that relates optimal policy probabilities under entropy regularization
to softmax consistent state values for \emph{any} action sequence;
second, we use this result to formulate a novel optimization objective
that allows for a stable form of off-policy actor-critic learning;
finally, we observe that under this objective the actor and critic can
be unified in a single model that coherently fulfills both roles.

\section{Notation \& Background}
\label{prelim}

We model an agent's behavior by a parametric distribution
$\pitheta(a \mid s)$ defined by a neural network over a finite set of
actions.  At iteration $t$, the agent encounters a state $s_t$ and
performs an action $a_t$ sampled from $\pitheta(a \mid s_t)$.  The
environment then returns a scalar reward $r_t$ and transitions to the
next state $s_{t+1}$.  

{\bf Note}: 
Our main results identify specific
properties that hold for arbitrary action sequences.
To keep the presentation clear and focus attention on the key properties, 
we provide a simplified presentation in the main body of this paper
by assuming \emph{deterministic} state dynamics.
This restriction is not necessary, and in 
Appendix~\ref{app:consist}
we provide a full treatment of the same concepts generalized to
stochastic state dynamics.
All of the desired properties continue to hold in the general case
and the algorithms proposed remain unaffected.


For simplicity, we
assume
the per-step reward $r_t$ and the next state $s_{t+1}$ are given by 
functions $r_t=r(s_t, a_t)$ and $s_{t+1}=f(s_t, a_t)$
specified by the environment. 
We begin the formulation by reviewing the key elements of
\mbox{Q-learning}~\cite{qlearning,watkins1992q}, 
which uses a notion of hard-max Bellman backup to enable off-policy TD control. 
First, observe that the expected discounted reward objective,
$\mro(s, \pi)$, can be recursively expressed as,
\begin{eqnarray}
  \mro(s, \pi) ~=~ \sum_a \pi(a \mid s)\,[r(s, a) + \gamma \mro(s', \pi)]~,~~~~~~~~~\text{where}~s'=f(s,a)~.
\end{eqnarray}
Let $\vop(s)$ denote the optimal state value at a state $s$ 
given by the maximum value of $\mro(s, \pi)$ over policies,
\ie~$\vop(s) = \mathrm{max}_\pi \mro(s, \pi)$.
Accordingly, let $\piop$ denote the optimal policy that results in $\vop(s)$
(for simplicity, assume there is one unique optimal policy), 
\ie~$\piop = \operatorname{argmax}_\pi \mro(s, \pi)$. 
Such an optimal policy is a one-hot distribution that assigns a probability 
of $1$ to an action with maximal return and $0$ elsewhere.
\comment{, \ie~$\piop(a \mid s) = \one{a =
    \operatorname{argmax}_a (r(s,a) + \gamma \vop(\pi,s'))}$} 
Thus we have
\begin{equation}
\vop(s) ~=~ \mro(s, \piop) = \max_a (r(s,a) + \gamma \vop(s')).
\label{eq:hard_state}
\end{equation}
This is the well-known
hard-max Bellman temporal consistency.
Instead of state values, one can equivalently (and more commonly)
express this consistency in terms of optimal action values, $\qop$:
\begin{equation}
\label{eq:hard}  \qop(s, a) ~=~ r(s, a) + \gamma\max_{a'} \qop(s^\prime, a')~.
\end{equation}
Q-learning relies on a value iteration algorithm based on
\eqref{eq:hard}, where 
$Q(s, a)$ is bootstrapped
based on successor action values $Q(s', a')$. 
\comment{
Previous work~\cite{watkins1992q,jaakkola1994} has proved that such 
hard-max iterative backups converge
to a unique fixed point that
maximizes $\mro(s, \pi)$ for each state $s$, 
when tabular action values are used. 
More recent work~\cite{dqn, atari} has successfully applied Q-learning to more
difficult tasks requiring function approximation using neural networks.}

\comment{

Accordingly, the $\mro$-optimal
state value of $s_0$ based on $\{\svop_1, \dots, \svop_n\}$ is given
by
\begin{equation}
\svop_0 = \mro(\piop) = \max_i (r_i + \gamma \svop_i).
\end{equation}
Thus, we have derived the hard-max
temporal consistency property introduced in \eqref{eq:hard2}.

\begin{eqnarray}
  \mro(s, \pi) &=& \sum_a \pi(a \mid s)\,\sum_{s', r} p(s', r \mid s, a)\,[r(s, a) + \gamma \vop(s')]\\
   &=& \sum_a \pi(a \mid s)\,\expected_{(s',\,r)}\,[r(s, a) + \gamma \vop(s')]~.
\end{eqnarray}

\begin{eqnarray}
  \label{eq:hard2} \vop(s) ~=~ \mro(\piop, s) &=& \max_{a} \sum_{s', r} p(s', r \mid s, a)\,[r(s, a) + \gamma \vop(s')]\\
&=& \max_{a} \,\expected_{(s',\,r)}\,[r(s, a) + \gamma \vop(s')]~.
\end{eqnarray}
}

\section{Softmax Temporal Consistency}
\label{method}

In this paper, we study the optimal state and action values for a 
\emph{softmax} form of temporal consistency~\cite{ziebart2008maximum,ziebart2010modeling,fox},
which arises by augmenting the standard expected reward objective 
with a \emph{discounted entropy regularizer}. 
Entropy regularization \cite{williams1991function} 
encourages exploration and
helps prevent early convergence to sub-optimal policies, 
as has been confirmed in practice (\eg~\cite{mnih2016asynchronous,urex}).  
In this case, one can express regularized expected reward as a sum of
the expected reward and a discounted entropy term,
\begin{equation}
\ento(s, \pi) ~=~ \mro(s, \pi) + \temp \HH{\gamma}(s, \pi)~,
\end{equation}
where $\tau \ge 0$ is a user-specified temperature parameter that
controls the degree of entropy regularization,
and the discounted entropy $\HH{\gamma}(s, \pi)$ is recursively defined as
\begin{equation}
\HH{\gamma}(s, \pi) ~=~ \sum_a \pi(a\mid s) \,[-\log \pi(a\mid s) + \gamma\,\HH{\gamma}(s', \pi)]~.
\label{eq:entropy}
\end{equation}
The objective $\ento(s, \pi)$ can then be re-expressed recursively as,
\begin{equation}
\ento(s, \pi) ~=~ \sum_a \pi(a\mid s) \,[ r(s,a) -\temp \log \pi(a\mid s) + \gamma \ento(s', \pi)]~.
\label{eq:objent}
\end{equation}
Note that when $\gamma=1$ this is equivalent to the 
entropy regularized objective proposed in~\cite{williams1991function}.

Let $\Vstar(s) = \mathrm{max}_\pi \ento(s,\pi)$ denote the soft
optimal state value at a state $s$ and let $\pistar(a \mid s)$ denote
the optimal policy at $s$ that attains the maximum of $\ento(s, \pi)$. 
When $\temp > 0$, the optimal policy is no longer a one-hot
distribution, since the entropy term prefers the use of
policies with more uncertainty. 
We characterize the optimal policy
$\pistar(a \mid s)$ in terms of the $\ento$-optimal state values of
successor states $\Vstar(s')$ as a Boltzmann distribution of the form,
\begin{equation}
\pistar(a \mid s) ~\propto~ \exp\{(r(s,a) + \gamma \Vstar(s')) /\tau\}~.
\label{eq:pistar}
\end{equation}
It can be verified that this is the solution by noting that the $\ento(s, \pi)$ 
objective is simply a $\tau$-scaled constant-shifted KL-divergence between $\pi$
and $\pistar$, hence the optimum is achieved when $\pi = \pistar$.

To derive $\Vstar(s)$ in terms of $\Vstar(s')$, 
the policy $\pistar(a \mid s)$ can be substituted into \eqref{eq:objent}, 
which after some manipulation yields the intuitive definition of optimal state 
value in terms of a softmax (\ie\ log-sum-exp) backup,
%
%
\begin{equation}
\Vstar(s) ~=~ \ento(s, \pistar) ~=~ \tau\log \sum_a \exp\{(r(s,a) + \gamma
\Vstar(s'))/\tau\}~.
\label{eq:vstar}
\end{equation}
Note that in the $\tau\to0$ limit one recovers the hard-max state
values defined in \eqref{eq:hard_state}. 
Therefore we can equivalently state softmax temporal consistency 
in terms of optimal action values $\qs(s, a)$ as,
\begin{equation}
\label{eq:soft}
 \qs(s, a) ~=~ r(s, a) + \gamma \Vstar(s') ~=~ r(s, a) + \gamma\tau \log\sum\nolimits_{a^\prime}
\exp(\qs(s^\prime, a^\prime) / \tau)~.
\end{equation}
Now, much like Q-learning, the consistency equation \eqref{eq:soft} 
can be used to perform one-step backups to asynchronously bootstrap 
$\qs(s, a)$ based on $\qs(s', a')$. 
In 
Appendix~\ref{app:consist} 
we prove that such a procedure,
in the tabular case,
converges to a unique fixed point representing the optimal values.

We point out that the notion of softmax Q-values has been studied 
in previous work (\eg~\cite{ziebart2010modeling, ziebart2008maximum,huang2015approximate,azar,mellowmax,fox}).
Concurrently to our work, \cite{haarnojaetal17} has also proposed a 
soft Q-learning algorithm for continuous control that is based on a 
similar notion of softmax temporal consistency.
However, we contribute new observations below that lead to 
the novel training principles we explore.

\section{Consistency Between Optimal Value \& Policy}

We now describe the main technical contributions of this paper,
which lead to the development of two novel
off-policy RL algorithms in \secref{sec:alg}.
The first key observation is that,
for the softmax value function $\Vstar$ in \eqref{eq:vstar},
the quantity $\exp\{\Vstar(s) / \tau\}$ also serves as the normalization factor 
of the optimal policy $\pistar(a \mid s)$ in \eqref{eq:pistar}; that is,
\begin{equation}
\pistar(a \mid s) ~=~ \frac{\exp\{(r(s,a) + \gamma \Vstar(s')) /\tau\}}{\exp\{\Vstar(s) / \tau\}} ~.
\label{eq:fraction-pistar}
\end{equation}
Manipulation of \eqref{eq:fraction-pistar} by taking the $\log$ of both sides
then reveals an important connection between the optimal state value 
$\Vstar(s)$,
the value $\Vstar(s')$ 
of the successor state $s'$ reached from \emph{any} action $a$ taken in $s$,
and the corresponding action probability under the 
optimal log-policy, $\log\pistar(a \mid s)$.

\begin{theorem}
  \label{thm:1}
\comment{The optimal policy $\pistar$ for $\ento(s_0, \theta)$ and the
  state values $\vstar$ defined in~\eqref{eq:vstar} satisfy} 
For $\temp\! >\! 0$, the policy $\pistar$ that maximizes $\ento$
and state values $\Vstar(s)\! =\! max_\pi \ento(s, \pi)$ 
satisfy the following temporal consistency property
for any state $s$ and action $a$ (where $s^\prime\!=\!f(s,a)$),
\begin{equation}
\Vstar(s) - \gamma \Vstar(s^\prime) ~=~ r(s, a)-\tau\log\pistar(a \mid s)~.
\label{eq:v pi consistent}
\end{equation}
\end{theorem}

\begin{proof}
\em
All theorems are established for the general case 
of a stochastic environment and discounted infinite horizon problems
in Appendix~\ref{app:consist}.
Theorem~\ref{thm:1} follows as a special case.
\end{proof}

Note that one can also characterize $\pistar$ in terms of $\qs$ as
\begin{equation}
\pistar(a \mid s) ~=~ \exp\{(\qs(s, a) - \vstar(s)) / \tau\}~.
\end{equation}

An important property of the one-step softmax consistency established 
in \eqref{eq:v pi consistent} is that it can be extended to a 
\emph{multi-step} consistency defined on any 
action \emph{sequence} from any given state.
That is, the softmax optimal state values at
the beginning and end of any action sequence can be related to the
rewards and optimal log-probabilities observed along the trajectory.

\begin{corollary}
\label{cor:2}
For $\temp>0$, the optimal policy $\pistar$ and optimal state values $\Vstar$
satisfy the following extended 
temporal consistency property,
for any state $s_1$ and any action sequence $a_1,...,a_{t-1}$
(where $s_{i+1} = f(s_i, a_i)$):
\begin{equation}
  \Vstar(s_1) - \gamma^{t-1}\Vstar(s_t) 
  ~=~ 
  \sum_{i=1}^{t-1} \gamma^{i-1} 
  [r(s_i, a_i) - \temp \log \pistar(a_i \mid s_i)]
~.
\label{eq:pathwise}
\end{equation}
\end{corollary}

\begin{proof}
\em
The proof in 
Appendix~\ref{app:consist} 
applies 
(the generalized version of) \theoref{thm:1} to any $s_1$ 
and sequence $a_1,...,a_{t-1}$,
summing the left and right hand sides of (the generalized version of)
\eqref{eq:v pi consistent} 
to induce telescopic cancellation of intermediate state values.
Corollary~\ref{cor:2} follows as a special case.
\end{proof}

Importantly, the converse of~\theoref{thm:1} (and Corollary~\ref{cor:2})
also holds:

\begin{theorem}
\label{thm:2}
If a policy $\pi(a\mid s)$ and state value function $V(s)$
satisfy the consistency property \eqref{eq:v pi consistent}
for all states $s$ and actions $a$ (where $s^\prime = f(s,a)$),
then $\pi=\pistar$ and $V=\vstar$.
(See 
Appendix~\ref{app:consist}.)
\end{theorem}


\theoref{thm:2} motivates the use of one-step and multi-step
path-wise consistencies as the foundation of RL algorithms that
aim to learn parameterized policy and value estimates by minimizing
the discrepancy between the left and right hand sides of 
\eqref{eq:v pi consistent} and \eqref{eq:pathwise}.

\section{\longname (\shortname)}
\label{sec:alg}

The temporal consistency properties
between the optimal policy and optimal state values developed above 
lead to a natural path-wise objective for training a policy 
$\pi_\theta$, parameterized by $\theta$, and a state value function $\Vphi$,
parameterized by $\phi$,
via the minimization of a soft consistency error. 
Based on \eqref{eq:pathwise}, we first define a notion of soft
consistency for a $d$-length sub-trajectory $s_{i:i+d}\equiv(s_i, a_i, \ldots,
s_{i+d-1}, a_{i+d-1}, s_{i+d})$ as a function of $\theta$ and $\phi$:
\begin{equation}
\CC{\gamma}(s_{i:i+d}, \theta, \phi) = 
-\Vphi(s_i) + \gamma^{d}\Vphi(s_{i+d})
+ \sum\nolimits_{j=0}^{d-1} \gamma^{j} [r(s_{i+j}, a_{i+j}) - \temp \log \pi_\theta(a_{i+j} \mid s_{i+j})]~.
\end{equation}
\comment{ This definition may be extended to trajectories that
  terminate before step $t$ by considering all rewards and
  log-probabilities after a terminal state as $0$.}  
The goal of a learning algorithm can then be to find $\Vphi$ and $\pi_\theta$ 
such that $\CC{\gamma}(s_{i:i+d}, \theta, \phi)$ is as close to $0$ as possible
for all sub-trajectories $s_{i:i+d}$. 
Accordingly, we propose a new learning algorithm, 
called \emph{\longname (\shortname)}, 
that attempts to minimize the squared soft consistency error 
over a set of sub-trajectories $\E$,
\begin{equation}
\pclo(\theta,\phi) = \sum_{s_{i:i+d} \in \E}\frac{1}{2}\CC{\gamma}(s_{i:i+d},{\theta},{\phi})^2.
\label{eq:objcac}
\end{equation}

The \shortname update rules for $\theta$ and $\phi$ are derived by 
calculating the gradient of~\eqref{eq:objcac}. 
For a given trajectory $s_{i:i+d}$ these take the form,
\begin{eqnarray}
\label{eq:piupd}
\Delta\theta &=& \eta_\pi\,\CC{\gamma}(s_{i:i+d},{\theta},{\phi}) \, \sum\nolimits_{j=0}^{d-1} \gamma^{j}  \nabla_\theta \log \pi_\theta(a_{i+j} \mid s_{i+j})~,\\
\label{eq:phiupd}
\Delta\phi &=& \eta_v\,\CC{\gamma}(s_{i:i+d},{\theta},{\phi}) \left(\nabla_\phi\Vphi(s_i)
  - \gamma^{d}\nabla_\phi\Vphi(s_{i+d})\right)~,
\end{eqnarray}
where $\eta_v$ and $\eta_\pi$ denote the value and policy learning rates
respectively. 
Given that the consistency property must hold on \emph{any} path, 
the \shortname algorithm applies the updates 
\eqref{eq:piupd} and \eqref{eq:phiupd} both
to trajectories sampled on-policy from $\pi_\theta$ 
as well as trajectories sampled from a replay buffer. 
The union of these trajectories comprise the set $\E$ used in
$\eqref{eq:objcac}$ to define $\pclo$.

Specifically, given a fixed rollout parameter $d$, at each iteration, 
\shortname samples a batch of on-policy trajectories and computes the 
corresponding
parameter updates for each sub-trajectory of length $d$.  
Then \shortname exploits off-policy trajectories by maintaining a replay 
buffer and applying additional updates based on a batch of episodes sampled
from the buffer at each iteration.
We have found it beneficial to sample replay episodes proportionally
to exponentiated reward, mixed with a
uniform distribution, although we did not exhaustively experiment with
this sampling procedure. 
In particular, we sample a full episode
$s_{0:T}$ from the replay buffer of size $B$ with probability $0.1/B +
0.9\cdot\exp(\alpha \sum_{i=0}^{T-1}r(s_i, a_i))/Z$, where we use no discounting
on the sum of rewards, $Z$ is a normalization factor, and $\alpha$ is a
hyper-parameter. Pseudocode of \shortname is provided in the
\supp.

We note that in stochastic settings, our 
squared inconsistency objective approximated by Monte Carlo samples
is a biased estimate of the true squared inconsistency 
(in which an expectation over stochastic dynamics occurs inside rather
than outside the square).
This issue arises in Q-learning as well, and others have proposed possible
remedies which can also be applied to \shortname~\cite{antos2008learning}.

\comment{While we focused our experiments on environments with relatively
short episodes (length at most $100$), in environments with longer
episodes the algorithm may be altered to be applied on sub-episodes
of manageable length.  It may also be beneficial to use multiple 
rollout lengths $d$ and optimize consistency at multiple scales,
although we did not explore this.}

\subsection{\longuname (\shortuname)}

The \shortname algorithm maintains a separate model for the policy and
the state value approximation.
However, given the soft consistency between the state and action value 
functions (\eg in \eqref{eq:soft}),
one can express the soft consistency errors strictly in terms of Q-values. 
Let $\Qrho$ denote a model of action values parameterized by
$\rho$, based on which one can estimate both the state values and the policy as,
\begin{eqnarray}
\Vrho(s) &=& \tau \log \sum\nolimits_a \exp\{\Qrho(s, a) / \tau\}~,\\
\pirho(a \mid s) &=& \exp\{ (\Qrho(s, a) - \Vrho(s)) / \tau\}~.
\end{eqnarray}
Given this unified parameterization of policy and value, 
we can formulate an alternative algorithm,
called \emph{\longuname (\shortuname)}, 
which optimizes the same objective (\ie~\eqref{eq:objcac}) as \shortname 
but differs by combining the policy and value function into a single model. 
Merging the policy and value function models in this way is
significant because it presents a new actor-critic paradigm where
the policy (actor) is not distinct from the values (critic).  
We note that in practice, we have found it beneficial to apply updates to $\rho$
from $\Vrho$ and $\pirho$ using different learning rates, very much
like \shortname.  Accordingly, the update rule for $\rho$ takes the form,
\begin{eqnarray}
\Delta\rho &=& 
\eta_\pi \CC{\gamma}(s_{i:i+d},\rho) \,
\sum\nolimits_{j=0}^{d-1} \gamma^{j} \nabla_\rho \log \pi_\rho(a_{i+j} \mid s_{i+j}) + \\
& & \eta_v \CC{\gamma}(s_{i:i+d},\rho)\left(\nabla_\rho\Vrho(s_i) - \gamma^{d}\nabla_\rho\Vrho(s_{i+d})\right)~.
\label{eq:rhoupd}
\end{eqnarray}

\subsection{Connections to Actor-Critic and Q-learning}

To those familiar with advantage-actor-critic 
methods~\citep{mnih2016asynchronous} (A2C and its asynchronous analogue A3C) 
\shortname's update rules might appear to be similar.
In particular, advantage-actor-critic is an on-policy method that exploits
the expected value function,
\begin{equation}
\vpi(s) ~=~ \sum\nolimits_a \pi(a\mid s) \,[r(s, a) + \gamma\vpi(s')]~,
\end{equation}
to reduce the variance of policy gradient, 
in service of maximizing the expected reward.
As in \shortname, two models are trained concurrently: 
an actor $\pitheta$ that determines the policy, 
and a critic $\Vphi$ that is trained to estimate $V^{\pitheta}$.  
A fixed rollout parameter $d$ is chosen, 
and the advantage of an on-policy trajectory $s_{i:i+d}$ is estimated by 
\begin{equation}
  \Ad{\gamma}(s_{i:i+d},\phi) = 
- \Vphi(s_{i}) + \gamma^d\Vphi(s_{i+d}) 
+ \sum\nolimits_{j=0}^{d-1} \gamma^{j} r(s_{i+j},a_{i+j})~.
\end{equation}
The advantage-actor-critic updates for $\theta$ and $\phi$ can 
then be written as,
\begin{eqnarray}
\Delta\theta &=& \eta_\pi \expected_{s_{i:i+d} | \theta}
\left[\Ad{\gamma}(s_{i:i+d},{\phi})\nabla_\theta\log\pitheta(a_i|s_i)\right]~,\\
\Delta\phi &=& \eta_v \expected_{s_{i:i+d} | \theta}
\left[\Ad{\gamma}(s_{i:i+d},{\phi})\nabla_\phi\Vphi(s_i)\right]~,
\label{eq:piupd2}
\end{eqnarray}
where the expectation $\expected_{s_{i:i+d} | \theta}$ denotes sampling
from the current policy~$\pitheta$.
These updates exhibit a striking similarity to the updates expressed in
~\eqref{eq:piupd} and ~\eqref{eq:phiupd}.
In fact, if one takes \shortname with $\tau\to0$ and omits the replay buffer, 
a slight variation of A2C is recovered.  
In this sense, one can interpret \shortname as a generalization of A2C.
Moreover, while A2C is restricted to on-policy samples,
\shortname minimizes an inconsistency measure that is defined on any path, 
hence it can exploit replay data to enhance its efficiency
via off-policy learning.

It is also important to note that for A2C, 
it is essential that $\Vphi$ tracks the non-stationary target $V^{\pitheta}$ 
to ensure suitable variance reduction.
In \shortname, no such tracking is required.
This difference is more dramatic in \shortuname,
where a single model is trained both as an actor and a critic.
That is, it is not necessary to have a separate actor and critic; 
the actor itself can serve as its own critic.

One can also compare \shortname to hard-max temporal consistency RL algorithms,
such as Q-learning~\citep{qlearning}.
In fact, setting the rollout to $d=1$ in \shortuname leads to a form of 
soft Q-learning, with the degree of softness determined by $\tau$.
We therefore conclude that the path consistency-based algorithms
developed in this paper also generalize Q-learning.  
Importantly, \shortname and \shortuname
are not restricted to single step consistencies, 
which is a major limitation of Q-learning.
While some have proposed using multi-step backups for
hard-max Q-learning~\citep{peng1996incremental, mnih2016asynchronous},
such an approach is not theoretically sound, 
since the rewards received after a non-optimal action do not relate to the 
hard-max Q-values $\qop$.
Therefore, one can interpret the notion of temporal consistency proposed
in this paper as a sound generalization of the one-step temporal consistency
given by hard-max Q-values.

\section{Related Work}
\label{related}

Connections between softmax Q-values and optimal entropy-regularized 
policies have been previously noted.
In some cases entropy regularization is expressed in the form
of relative entropy \citep{azaretal11,azaretal12,fox,schulmanetal17},
and in other cases it is the standard entropy~\citep{ziebart2010modeling}.
While these papers derive similar relationships to 
\eqref{eq:pistar} and \eqref{eq:vstar},
they stop short of stating the single- and multi-step consistencies
over all action choices we highlight.
Moreover, the algorithms proposed in those works
are essentially single-step Q-learning variants,
which suffer from the limitation of using single-step backups.
Another recent work~\citep{pgq2017}
uses the softmax relationship in the limit of $\tau\to0$
and proposes to
augment an actor-critic algorithm with offline updates that minimize
a set of single-step hard-max Bellman errors. 
Again, the methods we propose are differentiated by
the multi-step path-wise consistencies which allow the resulting
algorithms to utilize multi-step trajectories from off-policy
samples in addition to on-policy samples.

The proposed \shortname and \shortuname algorithms bear some similarity to
multi-step Q-learning~\citep{peng1996incremental}, which rather than minimizing
one-step hard-max Bellman error, optimizes a Q-value function approximator by
unrolling the trajectory for some number of steps before using a hard-max
backup.  While this method has shown some empirical
success~\cite{mnih2016asynchronous}, its theoretical justification is lacking,
since rewards received after a non-optimal action no longer relate to the 
hard-max Q-values $\qop$. 
In contrast, the algorithms we propose incorporate the
log-probabilities of the actions on a multi-step rollout, 
which is crucial for the version of softmax consistency we consider.

Other notions of temporal consistency similar to softmax consistency have
been discussed in the RL literature. 
Previous work has used a {\em Boltzmann weighted average} operator
\cite{littman,azar}. 
In particular, this
operator has been used by~\cite{azar} to propose an iterative algorithm
converging to the optimal maximum reward policy inspired by the work of
\cite{kappen2005path, todorov2006linearly}. While they use the Boltzmann
weighted average, they briefly mention that a softmax (log-sum-exp) operator
would have similar theoretical properties.  More recently~\cite{mellowmax}
proposed a mellowmax operator, defined as log-{\em average}-exp. These
log-average-exp operators share a similar non-expansion property, and
the proofs of non-expansion are related.  
Additionally it is possible to show
that when restricted to an infinite horizon setting, the fixed point of the
mellowmax operator is a constant shift of the $\qs$
investigated here.  
In all these cases,
the suggested training algorithm optimizes a single-step consistency unlike
\shortname and \shortuname, which optimizes a multi-step consistency.
Moreover, these papers do not present a clear relationship between the action
values at the fixed point and the entropy regularized expected reward
objective, which was key to the formulation and algorithmic development
in this paper.

Finally, there has been a considerable amount of work in reinforcement
learning using off-policy data to design more sample efficient
algorithms. Broadly speaking, these methods can be understood as trading off
bias \citep{sutton1999policy, silver14ddpg, lillicrap2015continuous,
  gu2016deep} and variance \citep{precup2000eligibility,
  munos2016safe}. Previous work that has considered multi-step
off-policy learning has typically used a correction (\eg~via
importance-sampling \citep{precup2001off} or truncated importance
sampling with bias correction \citep{munos2016safe}, or eligibility
traces \citep{precup2000eligibility}). By contrast, our method defines
an unbiased consistency for an entire trajectory applicable to on- and
off-policy data. 
An empirical comparison with all these methods 
remains however an interesting avenue for future work.

\section{Experiments}
\label{exp}

We evaluate the proposed algorithms, namely \shortname \& \shortuname, across several different tasks
and compare them to an A3C implementation, 
based on~\cite{mnih2016asynchronous},
and an implementation of double Q-learning with prioritized experience replay,
based on~\cite{pdqn}. 
We find that \shortname can consistently
match or beat the performance of these baselines.  
We also provide a comparison between \shortname 
and \shortuname and find that the use of a single
unified model for both values and policy can be competitive
with \shortname.

These new algorithms are easily amenable to incorporate expert
trajectories.  Thus, for the more difficult tasks we also experiment
with seeding the replay buffer with $10$ randomly sampled expert trajectories.
During training we ensure that these trajectories are not removed from 
the replay buffer and always have a maximal priority.

The details of the tasks and 
the experimental setup are provided in the \supp.

\begin{figure}[t!]
\begin{center}
  \begin{tabular}{@{}c@{}c@{}c@{}c@{}}
    \tiny Synthetic Tree & \tiny Copy & \tiny DuplicatedInput & \tiny RepeatCopy \\
    \includegraphics[width=0.21\columnwidth]{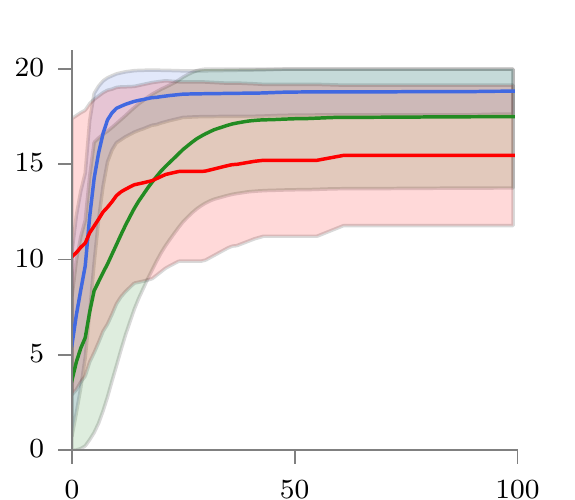} &
    \includegraphics[width=0.21\columnwidth]{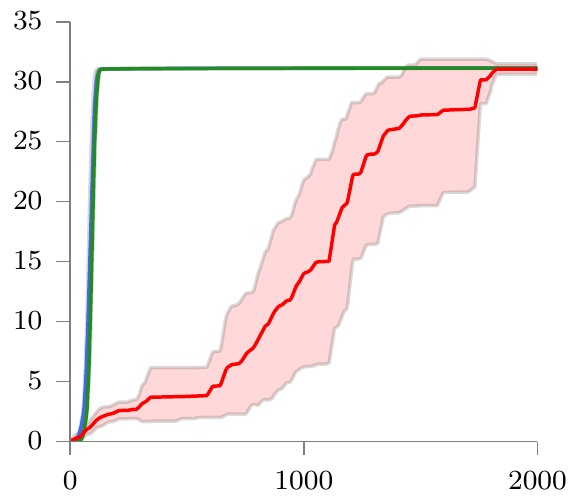} &
    \includegraphics[width=0.21\columnwidth]{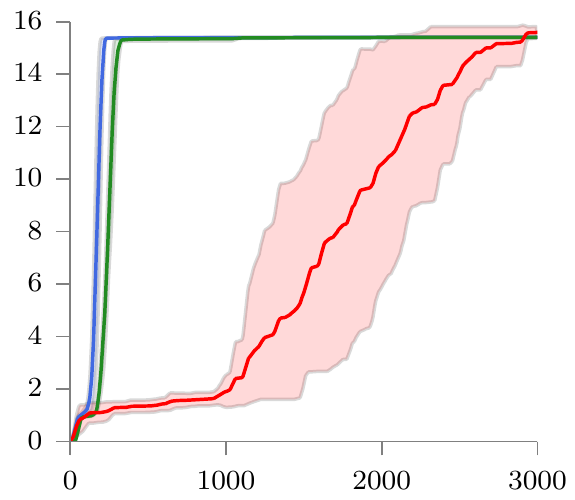} &
    \includegraphics[width=0.21\columnwidth]{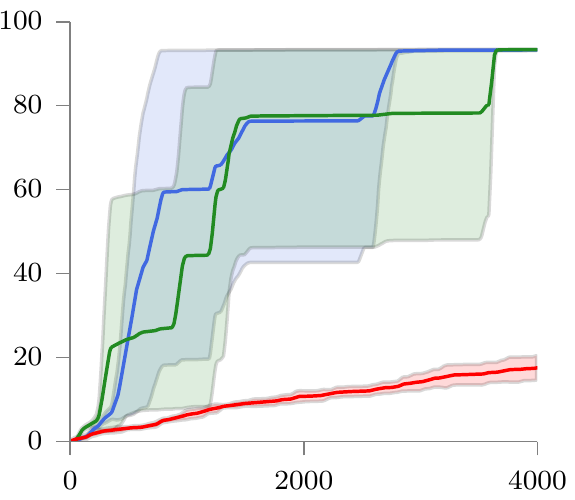} \\
    \tiny Reverse & \tiny ReversedAddition & \tiny ReversedAddition3 & \tiny Hard ReversedAddition \\
    \includegraphics[width=0.21\columnwidth]{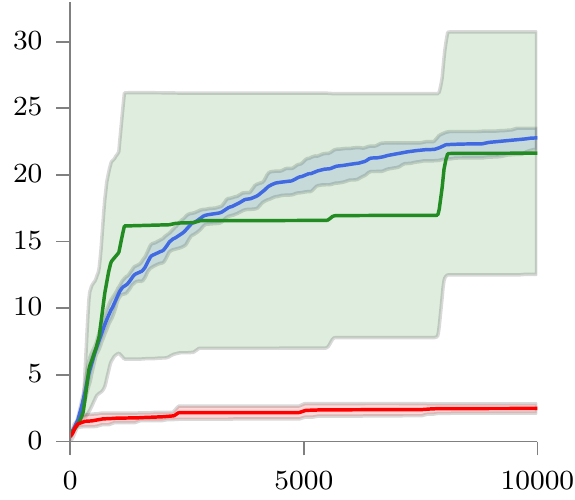} &
    \includegraphics[width=0.21\columnwidth]{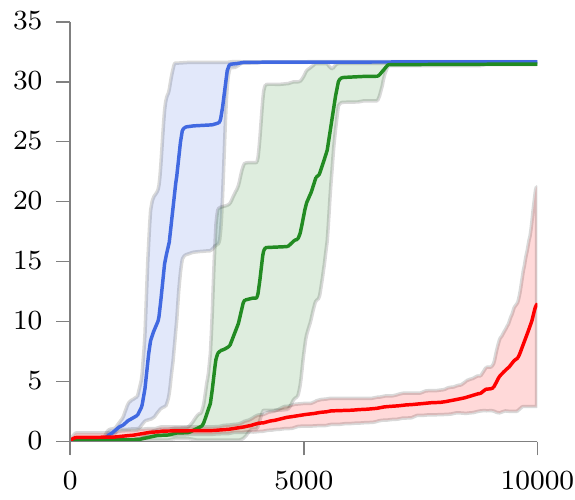} &
    \includegraphics[width=0.21\columnwidth]{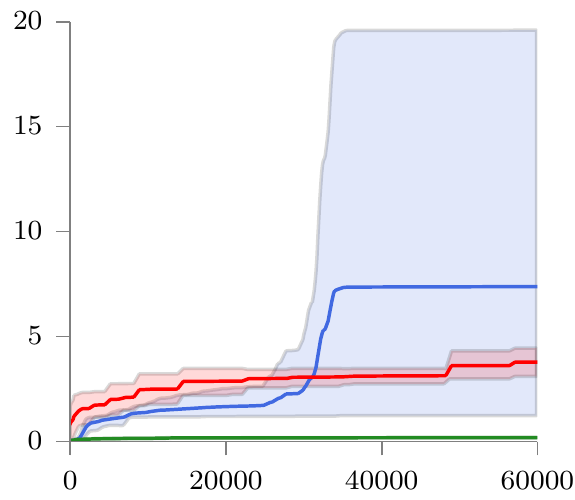} &
    \includegraphics[width=0.21\columnwidth]{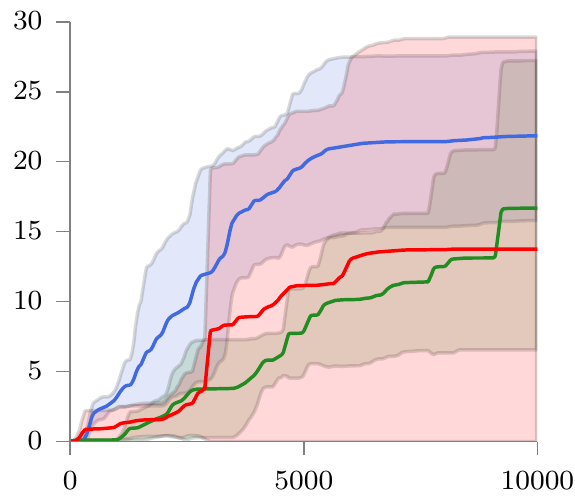} \\
    \multicolumn{4}{c}{\includegraphics[width=0.3\columnwidth]{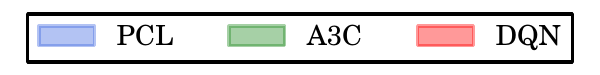}}
  \end{tabular}
\end{center}
\caption{
The results of \shortname against A3C and DQN baselines.
Each plot shows average reward across $5$ random training runs 
($10$ for Synthetic Tree) after choosing best hyperparameters.
We also show a single 
standard deviation bar clipped at the min and max.
The x-axis is number of training iterations.
\shortname exhibits comparable performance to A3C in some tasks,
but clearly outperforms A3C on the more challenging
tasks.  Across all tasks, the performance of DQN is worse than
\shortname.  
}
\label{fig:results}
\vspace{-0.4cm}
\end{figure}

\subsection{Results}
We present the results of each of the variants \shortname, A3C, and DQN
in~\figref{fig:results}.  After finding the best hyperparameters (see 
Section \ref{sec:implementation}), 
we plot the average reward over training iterations
for five randomly seeded runs. For the Synthetic Tree environment, the same
protocol is performed but with ten seeds instead.

The gap between \shortname and A3C is hard to discern in
some of the more simple tasks such as Copy, Reverse,
and RepeatCopy. However, a noticeable gap is observed in the
Synthetic Tree and DuplicatedInput results and more significant gaps
are clear in the harder tasks, including 
ReversedAddition, ReversedAddition3, and Hard ReversedAddition.
Across all of the experiments, it is clear that the prioritized DQN
performs worse than \shortname.
These results suggest that \shortname is a competitive  
RL algorithm, which in some cases
significantly outperforms strong baselines.

\begin{figure}[t!]
\begin{center}
  \begin{tabular}{@{}c@{}c@{}c@{}c@{}}
    \tiny Synthetic Tree & \tiny Copy & \tiny DuplicatedInput & \tiny RepeatCopy \\
    \includegraphics[width=0.21\columnwidth]{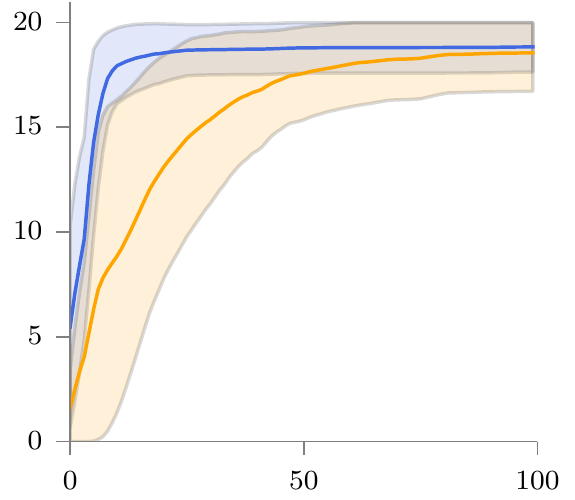} &
    \includegraphics[width=0.21\columnwidth]{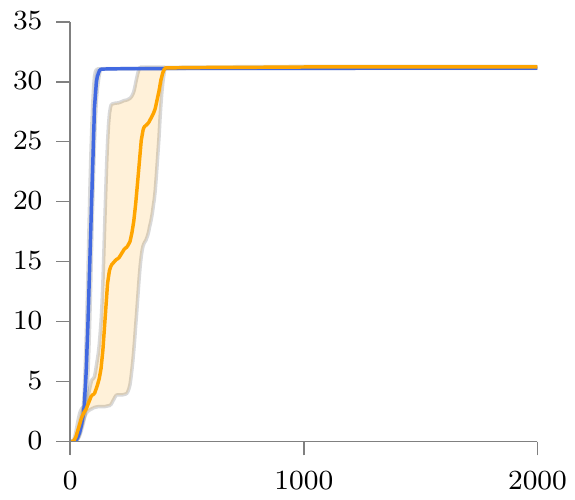} &
    \includegraphics[width=0.21\columnwidth]{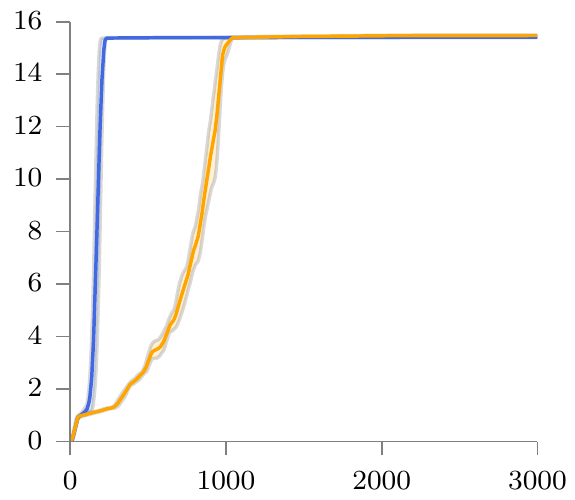} &
    \includegraphics[width=0.21\columnwidth]{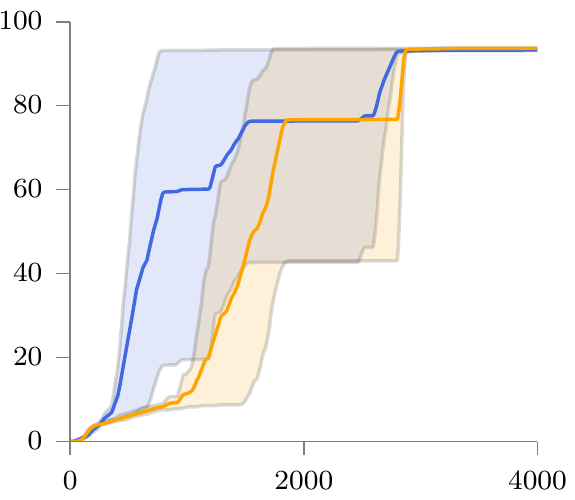} \\
    \tiny Reverse & \tiny ReversedAddition & \tiny ReversedAddition3 & \tiny Hard ReversedAddition \\
    \includegraphics[width=0.21\columnwidth]{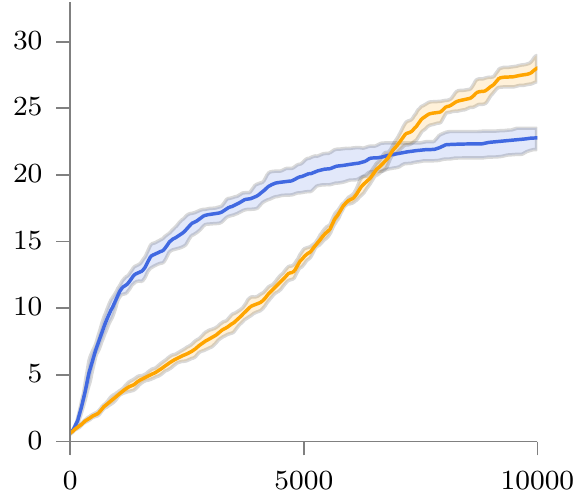} &
    \includegraphics[width=0.21\columnwidth]{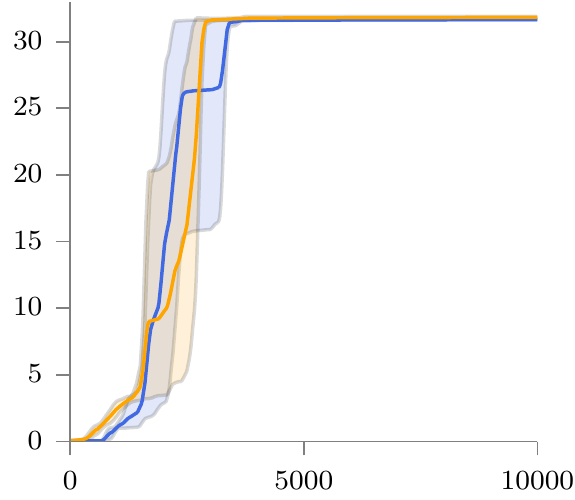} &
    \includegraphics[width=0.21\columnwidth]{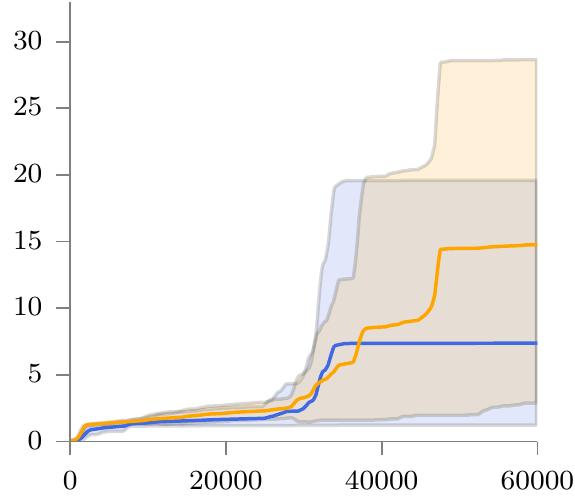} &
    \includegraphics[width=0.21\columnwidth]{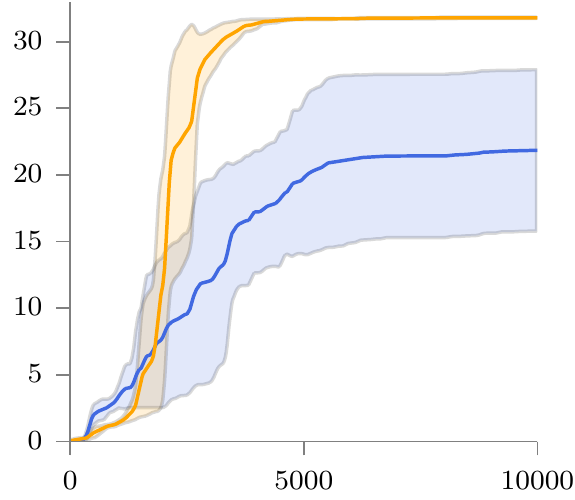} \\
    \multicolumn{4}{c}{\includegraphics[width=0.3\columnwidth]{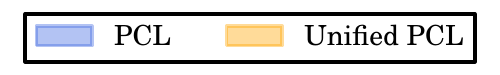}}
  \end{tabular}
\end{center}
\caption{
The results of \shortname~\vs~\shortuname.
Overall we find that using a single model for both values and
policy is not detrimental to training.  Although in some 
of the simpler tasks \shortname has an edge over \shortuname,
on the more difficult tasks, \shortuname preforms better.
}
\label{fig:uniresults}
\vspace{-0.4cm}
\end{figure}

We compare \shortname to \shortuname
in~\figref{fig:uniresults}.  The same protocol is performed
to find the best hyperparameters and plot the average reward over several
training iterations.  
We find that using a single model for both values and policy in \shortuname 
is slightly detrimental on the simpler tasks, 
but on the more difficult tasks \shortuname
is competitive or even better than \shortname.

\begin{figure}[t]
\begin{center}
  \begin{tabular}{@{}c@{}c@{}c@{}c@{}}
    \tiny Reverse & \tiny ReversedAddition & \tiny ReversedAddition3 & \tiny Hard ReversedAddition \\
    \includegraphics[width=0.21\columnwidth]{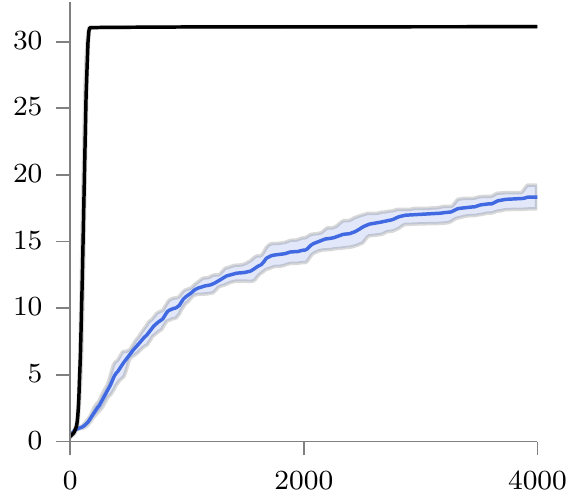} &
    \includegraphics[width=0.21\columnwidth]{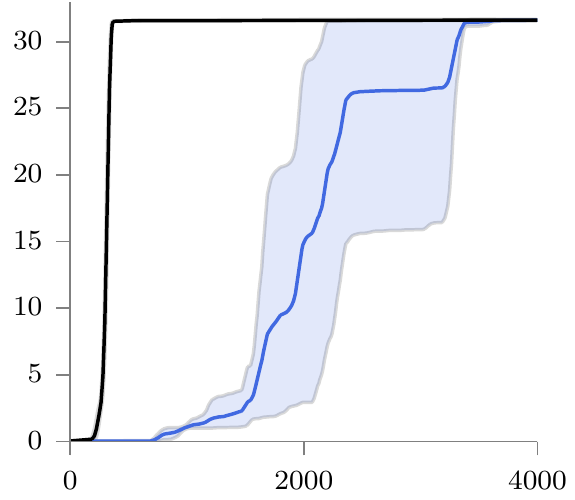} &
    \includegraphics[width=0.21\columnwidth]{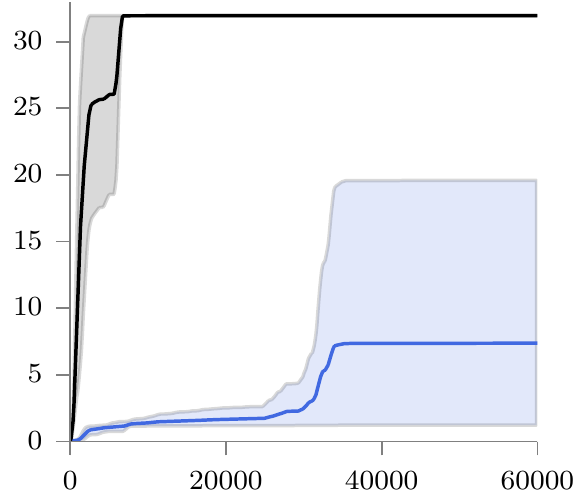} &
    \includegraphics[width=0.21\columnwidth]{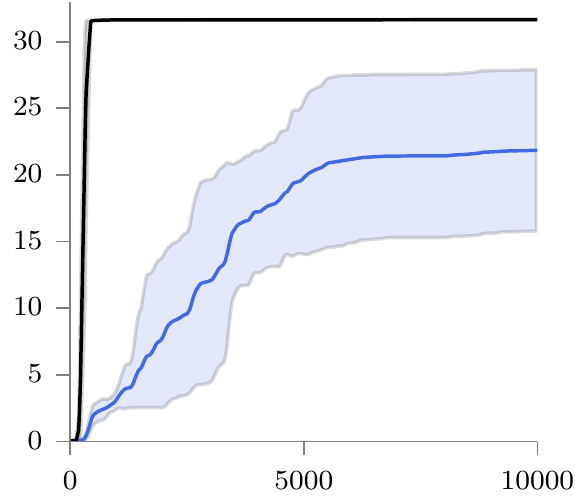} \\
    \multicolumn{4}{c}{\includegraphics[width=0.3\columnwidth]{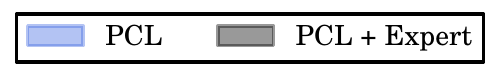}}
  \end{tabular}
\end{center}
\caption{
The results of \shortname~\vs~\shortname augmented with
a small number of expert trajectories on the hardest algorithmic tasks.
We find that incorporating expert trajectories
greatly improves performance.
}
\label{fig:exp_results}
\vspace{-0.5cm}
\end{figure}

We present the results of \shortname along with
\shortname augmented with expert trajectories in~\figref{fig:exp_results}.
We observe that the incorporation of expert trajectories
helps a considerable amount.  Despite only using a small number of expert trajectories
(\ie~$10$) as opposed to the mini-batch size of $400$,
the inclusion of expert trajectories in the training process significantly improves the agent's 
performance.  
We performed similar experiments with \shortuname and observed
a similar lift from using expert trajectories.
Incorporating expert trajectories in \shortname is relatively trivial
compared to the specialized methods developed for other 
policy based algorithms~\cite{abbeel2004apprenticeship, ho2016generative}.
While we did not compare to other
algorithms that take advantage of expert trajectories, 
this success shows the promise of using path-wise consistencies.
Importantly, the ability of \shortname to incorporate expert trajectories without
requiring adjustment or correction is a desirable property in real-world
applications such as robotics.

\section{Conclusion}
\label{conc}

We study the characteristics of the optimal policy and state values
for a maximum expected reward objective in the presence of 
\emph{discounted entropy regularization}. 
The introduction of an entropy regularizer induces
an interesting softmax consistency
between the optimal policy and optimal state values, which
may be expressed as either a single-step or multi-step consistency.
This softmax consistency leads to the development of \longname (\shortname),
an RL algorithm that resembles actor-critic in that it
maintains and jointly learns a model of the state values and a model
of the policy, and is similar to Q-learning in that it minimizes a
measure of temporal consistency error.  
We also propose the variant \shortuname which maintains a single model
for both the policy and the values, thus upending the
actor-critic paradigm of separating the actor from the critic.
Unlike standard policy based RL algorithms, \shortname and \shortuname apply 
to both on-policy and
off-policy trajectory samples. 
Further, unlike value based RL algorithms, \shortname and \shortuname
can take advantage of multi-step consistencies.
Empirically, \shortname and \shortuname exhibit a significant
improvement over baseline methods across several algorithmic
benchmarks.

\section{Acknowledgment}

We thank Rafael Cosman, Brendan O'Donoghue, Volodymyr Mnih, George
Tucker, Irwan Bello, and the Google Brain team for insightful comments
and discussions.

 
\bibliography{nips.bib}
\bibliographystyle{abbrv}

\newpage
\appendix

\section{Pseudocode}
Pseudocode for PCL is presented in
Algorithm~\ref{alg:cac}.

\begin{algorithm}[!htb]
\caption{\longname}
\label{alg:cac}    

\begin{algorithmic}
\STATE {\bfseries Input:} 
Environment $ENV$, 
learning rates $\eta_\pi,\eta_v$, discount factor $\gamma$,
rollout $d$, number of steps $N$, replay buffer capacity $B$, prioritized
replay hyperparameter $\alpha$.

\FUNCTION{Gradients($s_{0:T}$)}
\STATE \emph{// We use $\G{\gamma}(s_{t:t+d}, \pitheta)$ to denote a discounted sum of log-probabilities from $s_t$ to $s_{t+d}$.}
\STATE Compute $\Delta\theta = \sum_{t=0}^{T-d} \C{\gamma}{\theta}{\phi}(s_{t:t+d}) \nabla_\theta \G{\gamma}(s_{t:t+d}, \pitheta)$.
\STATE Compute $\Delta\phi=
\sum_{t=0}^{T-d} \C{\gamma}{\theta}{\phi}(s_{t:t+d}) \left(\nabla_\phi\Vphi(s_t) - \gamma^{d}\nabla_\phi\Vphi(s_{t+d})\right)$.

\STATE \emph{Return} $\Delta\theta, \Delta\phi$
\ENDFUNCTION

\STATE Initialize $\theta,\phi$.
\STATE Initialize empty replay buffer $RB(\alpha)$.
\FOR{$i=0$ {\bfseries to} $N-1$}
\STATE Sample $s_{0:T}\sim\pitheta(s_{0:})$ on $ENV$.
\STATE $\Delta\theta, \Delta\phi = \text{Gradients}(s_{0:T})$.
\STATE Update $\theta \leftarrow \theta + \eta_\pi\Delta\theta$.
\STATE Update $\phi \leftarrow \phi + \eta_V\Delta\phi$.

\STATE Input $s_{0:T}$ into $RB$ with priority $R^1(s_{0:T})$.
\STATE If $|RB| > B$, remove episodes uniformly at random.
\STATE Sample $s_{0:T}$ from $RB$.
\STATE $\Delta\theta, \Delta\phi = \text{Gradients}(s_{0:T})$.
\STATE Update $\theta \leftarrow \theta + \eta_\pi\Delta\theta$.
\STATE Update $\phi \leftarrow \phi + \eta_v\Delta\phi$.

\ENDFOR

\end{algorithmic}
\end{algorithm}


\section{Experimental Details}

We describe the tasks we experimented on as well as details of the 
experimental setup.

\subsection{Synthetic Tree}

As an initial testbed, we developed a simple synthetic environment.
The environment is defined by a 
binary decision tree of depth 20.  For each training run,
the reward on each edge is sampled uniformly from $[-1, 1]$
and subsequently normalized so that the 
maximal reward trajectory has total reward 20.  We trained using
a fully-parameterized model: for each node $s$ in the decision tree
there are two parameters to determine the logits of
$\pitheta(-|s)$ and
one parameter to determine $\Vphi(s)$.  In the Q-learning
and \shortuname
implementations only two parameters per node $s$ are needed
to determine the Q-values.

\subsection{Algorithmic Tasks}

For more complex environments, we evaluated \shortname, \shortuname, and the two baselines
on the algorithmic tasks from the OpenAI
Gym library~\citep{Brockman}.  This library provides six tasks,
in rough order of difficulty: 
Copy, DuplicatedInput, RepeatCopy, Reverse, ReversedAddition, and ReversedAddition3.
In each of these tasks, an agent operates on a grid of characters or digits,
observing one character or digit at a time.  At each time step, the agent may move one step
in any direction and optionally write a character or digit to output.  
A reward is received on each correct emission.
The agent's goal for each task is:
\begin{itemize}
\item \textbf{Copy:} Copy a $1\times n$ sequence of characters to output.
\item \textbf{DuplicatedInput:} Deduplicate a $1\times n$ sequence of characters.
\item \textbf{RepeatCopy:} Copy a $1\times n$ sequence of characters first in forward order, then reverse, and finally forward again.
\item \textbf{Reverse:} Copy a $1\times n$ sequence of characters in reverse order.
\item \textbf{ReversedAddition:} Observe two ternary numbers in little-endian order via a $2\times n$ grid and output their sum.
\item \textbf{ReversedAddition3:} Observe three ternary numbers in little-endian order via a $3\times n$ grid and output their sum.
\end{itemize}
These environments have an implicit curriculum associated with them.  
To observe the performance of our algorithm without curriculum,
we also include a task ``Hard ReversedAddition'' which has the same 
goal as ReversedAddition but does not utilize curriculum.

For these environments,
we parameterized the agent by a recurrent neural network with LSTM~\citep{lstm} cells of hidden dimension 128.

\subsection{Implementation Details}
\label{sec:implementation}

For our hyperparameter search, we found it simple
to parameterize the critic learning rate
in terms of the actor learning rate as 
$\eta_v = C \eta_\pi$, where $C$ is the \emph{critic weight}.

For the Synthetic Tree environment we used a batch size of 10, rollout of $d=3$,
discount of $\gamma=1.0$, and a replay buffer capacity of 10,000.  We fixed the
$\alpha$ parameter for \shortname's replay buffer to 1 and used $\epsilon=0.05$
for DQN.  To find the optimal hyperparameters, we performed an
extensive grid search over actor learning rate $\eta_\pi\in\{0.01, 0.05,
0.1\}$; critic weight $C\in\{0.1, 0.5, 1\}$; entropy regularizer
$\tau\in\{0.005, 0.01, 0.025, 0.05, 0.1, 0.25, 0.5, 1.0\}$ for A3C, \shortname, \shortuname;
and $\alpha\in\{0.1, 0.3, 0.5, 0.7, 0.9\},\beta\in\{0.2, 0.4, 0.6, 0.8, 1.0\}$
for DQN replay buffer parameters.  We used standard gradient descent
for optimization.

For the algorithmic tasks we used a batch size of 400,
rollout of $d=10$,
a replay buffer of capacity 100,000, ran using
distributed training with 4 workers, and fixed the 
actor learning rate $\eta_\pi$ to 0.005, which we found 
to work well across all variants.
To find the optimal hyperparameters, we performed an extensive
grid search over
discount $\gamma\in\{0.9, 1.0\}$, 
$\alpha\in\{0.1, 0.5\}$ for \shortname's replay buffer;
critic weight $C\in\{0.1, 1\}$;
entropy regularizer $\tau\in\{0.005, 0.01, 0.025, 0.05, 0.1, 0.15\}$;
$\alpha\in\{0.2, 0.4, 0.6, 0.8\}$, $\beta\in\{0.06, 0.2, 0.4, 0.5, 0.8\}$ for
the prioritized DQN replay buffer; and also experimented with exploration rates
$\epsilon\in\{0.05, 0.1\}$ and copy frequencies for the target DQN, $\{100, 200, 400,
600\}$.  In these experiments, we used the Adam
optimizer~\citep{adam}.

All experiments were implemented using Tensorflow~\citep{tensorflow}.


\section{Proofs}
\label{app:consist}

In this section, we provide a general theoretical foundation for this work,
including proofs of the main path consistency results.
We first
establish the basic results for a simple one-shot decision making setting.
These initial results will be useful in the proof of the general infinite
horizon setting.  

Although the main paper expresses the main claims under an assumption
of deterministic dynamics, this assumption is not necessary:
we restricted attention to the deterministic case in the main body
merely for clarity and ease of explanation.
Given that in this appendix we provide the general foundations for this work, 
we consider the more general stochastic setting throughout the later sections.

In particular, for the general stochastic, infinite horizon setting, 
we introduce and discuss the
entropy regularized expected return $O_{ENT}$ and define a ``softmax'' operator
$\calB^*$ (analogous to the Bellman operator for hard-max Q-values). 
We then show the existence of a unique fixed point $V^*$ of $\calB^*$, 
by establishing that the softmax  Bellman operator ($\calB^*$) is a 
contraction under the infinity norm.  
We then relate $V^*$ to the optimal value of the entropy regularized expected 
reward objective $\ento$, which we term $V^\dag$.
We are able to show that $V^* = V^\dag$, as expected.
Subsequently, we present a policy determined by $V^*$ that satisfies 
$V^*(s) = \ento(s, \pi^*)$.  
Then given the characterization of $\pi^*$ in terms of $V^*$, we
establish the consistency property stated in Theorem 1 of the main text.
Finally, we show that a consistent solution is optimal by satisfying the KKT
conditions of the constrained optimization problem (establishing Theorem 4 of
the main text).

\subsection{Basic results for one-shot entropy regularized optimization}

For $\tau>0$
and any vector $\vec{q}\in\mathbb{R}^n$, $n<\infty$,
define the scalar valued function $F_\tau$ (the ``softmax'') by
\begin{align}
F_\tau(\vec{q}) &= 
\tau\log\left(\sum_{a=1}^ne^{q_a/\tau}\right)
\end{align}
and define the vector valued function $\vec{f}_\tau$ 
(the ``soft indmax'')
by
\begin{align}
\vec{f}_\tau(\vec{q}) 
&= 
\frac{e^{\vec{q}/\tau}}{\sum_{a=1}^ne^{q_a/\tau}}
=
e^{(\vec{q}-F_\tau(\vec{q}))/\tau}
,
\end{align}
where the exponentiation is component-wise.
It is easy to verify that $\vec{f}_\tau=\nabla F_\tau$.
Note that $\vec{f}_\tau$ maps any real valued vector to a probability
vector.
We denote the probability simplex by
$\Delta=\{\vec{\pi}:\vec{\pi}\geq0,\vec{1}\cdot\vec{\pi}=1\}$,
and denote the entropy function by $H(\vec{\pi})=-\vec{\pi}\cdot\log\vec{\pi}$.

\begin{lemma}
\label{lem:a1}
\begin{align}
F_\tau(\vec{q})
    &= \max_{\vec{\pi}\in\Delta}\Bigl\{\vec{\pi}\cdot\vec{q}+\tau H(\vec{\pi})\Bigr\}
\label{eq:a1}
\\
&= \vec{f}_\tau(\vec{q})\cdot\vec{q}+\tau H(\vec{f}_\tau(\vec{q}))
\label{eq:a2}
\end{align}
\end{lemma}

\begin{proof}
First consider the constrained optimization problem on the right hand side of \eqref{eq:a1}.
The Lagrangian is given by
$L=\vec{\pi}\cdot(\vec{q}-\tau\log\vec{\pi})
+ \lambda(1-\vec{1}\cdot\vec{\pi})$,
hence
$\nabla L=\vec{q}-\tau\log\vec{\pi}-\tau-\lambda$.
The KKT conditions for this optimization problems are the following system of $n+1$ equations
\begin{align}
\vec{1}\cdot\vec{\pi} &=1
\label{eq:kkt1}
\\
\tau\log\vec{\pi} &= \vec{q}-v
\label{eq:kkt2}
\end{align}
for the $n+1$ unknowns, $\vec{\pi}$ and $v$, where $v=\lambda+\tau$.
Note that for any $v$, 
satisfying \eqref{eq:kkt2} requires
the unique assignment
$\vec{\pi}=\exp((\vec{q}-v)/\tau)$,
which also ensures $\vec{\pi}>0$.
To subsequently satisfy \eqref{eq:kkt1},
the equation $1=\sum_a\exp((q_a-v)/\tau)=e^{-v/\tau}\sum_a\exp(q_a/\tau)$
must be solved for $v$;
since the right hand side is strictly decreasing in $v$,
the solution is also unique and 
in this case
given by $v=F_\tau(\vec{q})$.
Therefore 
$\vec{\pi}=\vec{f}_\tau(\vec{q})$
and 
$v=F_\tau(\vec{q})$ 
provide the unique solution to the KKT conditions
\eqref{eq:kkt1}-\eqref{eq:kkt2}.
Since the objective is strictly concave,
$\vec{\pi}$ must be the unique global maximizer,
establishing \eqref{eq:a2}.
It is then easy to show
$F_\tau(\vec{q})
= \vec{f}_\tau(\vec{q})\cdot\vec{q}+\tau H(\vec{f}_\tau(\vec{q}))$
by algebraic manipulation, which establishes \eqref{eq:a1}.
\end{proof}

\begin{corollary}[Optimality Implies Consistency]
\label{cor:a1}
If $v^*=\displaystyle\max_{\vec{\pi}\in\Delta}\Bigl\{\vec{\pi}\cdot\vec{q}+\tau H(\vec{\pi})\Bigr\}$
then
\begin{align}
v^*&= q_a-\tau\log\pistar_a \;\mbox{ for all } a,
\end{align}
where $\vec{\pi}^*=\vec{f}_\tau(\vec{q})$.
\end{corollary}

\begin{proof}
From Lemma~\ref{lem:a1} we know
$v^*=F_\tau(\vec{q})=\vec{\pi}^*\cdot(\vec{q}-\tau\log\vec{\pi}^*)$
where $\vec{\pi}^*=\vec{f}_\tau(\vec{q})$.
From the definition of $\vec{f}_\tau$ it also follows that
$\log\pistar_a = (q_a-F_\tau(\vec{q}))/\tau$ for all $a$,
hence $v^*=F_\tau(\vec{q})=q_a-\tau\log\pistar_a$ for all $a$.
\end{proof}

\begin{corollary}[Consistency Implies Optimality]
\label{cor:a2}
If $v\in\mathbb{R}$ and $\vec{\pi}\in\Delta$ jointly satisfy
\begin{align}
v&= q_a-\tau\log\pi_a \;\mbox{ for all } a,
\label{eq:consistency}
\end{align}
then
$v=F_\tau(\vec{q})$ 
and 
$\vec{\pi}=\vec{f}_\tau(\vec{q})$;
that is, $\vec{\pi}$ must be an optimizer for \eqref{eq:a1}
and $v$ is its corresponding optimal value.
\end{corollary}

\begin{proof}
Any $v$ and $\vec{\pi}\in\Delta$ that jointly satisfy \eqref{eq:consistency}
must also satisfy the KKT conditions \eqref{eq:kkt1}-\eqref{eq:kkt2};
hence $\vec{\pi}$ must be the unique maximizer for \eqref{eq:a1}
and $v$ its corresponding objective value.
\end{proof}

Although these results are elementary, they reveal a strong connection
between optimal state values ($v$), optimal action values ($\vec{q}$)
and optimal policies ($\vec{\pi}$) under the softmax operators.
In particular, Lemma~\ref{lem:a1} states that,
if $\vec{q}$ is an optimal action value at some current state,
the optimal state value must be $v=F_\tau(\vec{q})$,
which is simply
the entropy regularized value of the optimal policy,
$\vec{\pi}=\vec{f}_\tau(\vec{q})$,
at the current state.

Corollaries \ref{cor:a1} and \ref{cor:a2} then make the stronger
observation that
this mutual consistency between the optimal state value,
optimal action values and optimal policy probabilities
must hold for \emph{every} action,
not just in expectation over actions sampled from $\vec{\pi}$;
and furthermore
that achieving mutual consistency in this form is \emph{equivalent}
to achieving optimality.

Below we will also need to make use of the following
properties of $F_\tau$.

\begin{lemma}
\label{lem:Fconj}
For any vector $\vec{q}$,
\begin{align}
F_\tau(\vec{q}) &=
\sup_{\vec{p}\in\Delta}
\Bigl\{\vec{p}\cdot\vec{q}-\tau\vec{p}\cdot\log\vec{p}\Bigr\}
.
\end{align}
\end{lemma}

\begin{proof}
Let $F^*_\tau$ denote the conjugate of $F_\tau$,
which is given by
\begin{align}
F_\tau^*(\vec{p})
&= \sup_{\vec{q}}\Bigl\{{\vec{q}\cdot\vec{p}-F_\tau(\vec{q})}\Bigr\}
= \tau\vec{p}\cdot\log\vec{p}
\end{align}
for $\vec{p}\in\textrm{dom}(F^*_\tau)=\Delta$.
Since $F_\tau$ is closed and convex, we also have that
$F_\tau=F^{**}_\tau$ \citep[Section~4.2]{borweinlewis00};
hence 
\begin{align}
F_\tau(\vec{q}) &=
\sup_{\vec{p}\in\Delta}\Bigl\{{\vec{q}\cdot\vec{p}-F^*_\tau(\vec{p})}\Bigr\}
.
\end{align}
\end{proof}

\begin{lemma}
\label{lem:qcontract}
For any two vectors $\vec{q}^{(1)}$ and $\vec{q}^{(2)}$,
\begin{align}
F_\tau(\vec{q}^{(1)}) - F_\tau(\vec{q}^{(2)})
&\leq 
\max_a 
\Bigl\{{q}^{(1)}_a - {q}^{(2)}_a\Bigr\}
.
\end{align}
\end{lemma}

\begin{proof}
Observe that by Lemma~\ref{lem:Fconj}
\begin{align}
F_\tau(\vec{q}^{(1)}) - F_\tau(\vec{q}^{(2)})
&=
\sup_{\vec{p}^{(1)}\in\Delta}\Bigl\{
\vec{q}^{(1)}\cdot\vec{p}^{(1)}-F^*_\tau(\vec{p}^{(1)})\Bigl\}
-
\sup_{\vec{p}^{(2)}\in\Delta}\Bigl\{
\vec{q}^{(2)}\cdot\vec{p}^{(2)}-F^*_\tau(\vec{p}^{(2)})\Bigr\}
\\
&=
\sup_{\vec{p}^{(1)}\in\Delta}\Bigl\{
\inf_{\vec{p}^{(2)}\in\Delta}\Bigl\{
\vec{q}^{(1)}\cdot\vec{p}^{(1)}
-
\vec{q}^{(2)}\cdot\vec{p}^{(2)}
-(F^*_\tau(\vec{p}^{(1)})
-F^*_\tau(\vec{p}^{(2)}))\Bigr\}\Bigr\}
\\
&\leq
\sup_{\vec{p}^{(1)}\in\Delta}\Bigl\{
\vec{p}^{(1)}\cdot(\vec{q}^{(1)} - \vec{q}^{(2)})\Bigr\}
\quad\mbox{ by choosing } \vec{p}^{(2)}=\vec{p}^{(1)}
\\
&\leq
\max_{a}\Bigl\{
{q}^{(1)}_a - {q}^{(2)}_a
\Bigr\}
.
\end{align}
\end{proof}

\begin{corollary}
\label{cor:qcontract}
$F_\tau$ is an $\infty$-norm contraction; that is,
for any two vectors $\vec{q}^{(1)}$ and $\vec{q}^{(2)}$,
\begin{align}
\left|
F_\tau(\vec{q}^{(1)}) - F_\tau(\vec{q}^{(2)})
\right|
&\leq 
\|\vec{q}^{(1)}-\vec{q}^{(2)}\|_\infty
\end{align}
\end{corollary}

\begin{proof}
Immediate from Lemma~\ref{lem:qcontract}.
\end{proof}

\subsection{Background results for \emph{on-policy} entropy regularized updates}

Although the results in the main body of the paper are expressed
in terms of deterministic 
problems,
we will prove that all the desired properties hold 
for the more general {\bf stochastic} case, 
where there is a stochastic transition $s,a\mapsto s'$
determined by the environment.
Given the characterization for this general case,
the application to the deterministic 
case is immediate.
We continue to assume
that the action space is finite, and that the state space is discrete.

For any policy $\pi$, define the entropy regularized expected return by
\begin{align}
\tilde V^\pi(s_\ell)
= \ento(s_\ell,\pi)
&=
\expected_{a_{\ell}s_{\ell+1}...|s_\ell}
\Bigg[
\sum_{i=0}^{\infty}
\gamma^i\big(
r(s_{\ell+i},a_{\ell+i})
-\tau\log\pi(a_{\ell+i}|s_{\ell+i})
\big)
\Bigg]
,
\label{eq:V^pi}
\end{align}
where the expectation is taken with respect to the policy $\pi$ and
with respect to the stochastic state transitions determined by the
environment.
We will find it convenient to also work with 
the on-policy Bellman operator defined by
\begin{align}
(\calB^\pi V)(s) 
&=
\expected_{a,s'|s}\Big[r(s,a)-\tau\log\pi(a|s)+\gamma V(s')\Big]
\\
&=
\expected_{a|s}
\Big[r(s,a)-\tau\log\pi(a|s)+\gamma\expected_{s'|s,a}\big[V(s')\big]\Big]
\\
&= 
\pi(:\!|s)\cdot(Q(s,:)-\tau\log\pi(:\!|s)),
\label{eq:q^pi}
\quad\mbox{ where }
\\[1ex]
Q(s,a) &= r(s,a) + \gamma\expected_{s'|s,a}[V(s')]
\label{eq:Qvdef}
\end{align}
for each state $s$ and action $a$.
Note that in \eqref{eq:q^pi} we are using $Q(s,:)$ to
denote a vector values over choices of $a$ for a given $s$,
and $\pi(:\!|s)$ to denote the vector of conditional action probabilities
specified by $\pi$ at state $s$.

\begin{lemma}
\label{lem:a2}
For any policy $\pi$ and state $s$,
$\tilde V^\pi(s)$ satisfies the recurrence
\begin{align}
\tilde V^\pi(s)
&=
\expected_{a|s}\Big[
r(s,a)+\gamma\expected_{s'|s,a}[\tilde V^\pi(s')]-\tau\log\pi(a|s)
\Big]
\\
&=
\pi(:\!|s)\cdot\left(\tilde{Q}^{\pi}(s,:)-\tau\log\pi(:\!|s)\right)
\mbox{ where }
\tilde{Q}^\pi(s,a)=r(s,a)+\gamma\expected_{s'|s,a}[\tilde V^\pi(s')]
\\
&=
(\calB^\pi\tilde V^\pi)(s)
.
\end{align}
Moreover, $\calB^\pi$ is a contraction mapping.
\end{lemma}

\begin{proof}
Consider an arbitrary state $s_\ell$.
By the definition of $\tilde V^\pi(s_\ell)$ in \eqref{eq:V^pi} we have
\begin{align}
\tilde V^\pi(s_\ell)
&=
\expected_{a_{\ell}s_{\ell+1}...|s_\ell}
\Bigg[
\sum_{i=0}^{\infty}
\gamma^i\big(
r(s_{\ell+i},a_{\ell+i})
-\tau\log\pi(a_{\ell+i}|s_{\ell+i})
\big)
\Bigg]
\\
&=
\expected_{a_{\ell}s_{\ell+1}...|s_\ell}
\Bigg[
r(s_{\ell},a_{\ell})
-\tau\log\pi(a_\ell|s_{\ell})
\\
&\hspace*{1.0in}
+\gamma
\sum_{j=0}^{\infty}
\gamma^{j}\big(
r(s_{\ell+1+j},a_{\ell+1+j})
-\tau\log\pi(a_{\ell+1+j}|s_{\ell+1+j})
\big)
\Bigg]
\nonumber
\\
&=
\expected_{a_{\ell}|s_\ell}
\Bigg[
r(s_{\ell},a_{\ell})
-\tau\log\pi(a_\ell|s_{\ell})
\\
&\hspace*{0.6in}
+\gamma
\expected_{s_{\ell+1}a_{\ell+1}...|s_\ell,a_\ell}
\Bigg[
\sum_{j=0}^{\infty}
\gamma^{j}\big(
r(s_{\ell+1+j},a_{\ell+1+j})
-\tau\log\pi(a_{\ell+1+j}|s_{\ell+1+j})
\big)
\Bigg]\Bigg]
\nonumber
\\
&=
\expected_{a_{\ell}|s_\ell}
\Big[
r(s_{\ell},a_{\ell})
-\tau\log\pi(a_\ell|s_{\ell})
+\gamma\expected_{s_{\ell+1}|s_\ell,a_\ell}[\tilde V^\pi(s_{\ell+1})]
\Big]
\\[1ex]
&=
\pi(:\!|s_\ell)\cdot\Big(\tilde{Q}^{\pi}(s_\ell,:)-\tau\log\pi(:\!|s_\ell)\Big)
\\
&=
(\calB^\pi\tilde V^\pi)(s_\ell)
.
\end{align}
The fact that $\calB^\pi$ is a contraction mapping follows directly from
standard arguments about the on-policy Bellman operator
\cite{tsitsiklisvanroy97}.
\end{proof}

Note that this lemma shows 
$\tilde V^\pi$ is a \emph{fixed point} of the corresponding on-policy
Bellman operator $\calB^\pi$.
Next, we characterize how quickly convergence to a fixed point is achieved
by repeated application of ther $\calB^\pi$ operator.

\begin{lemma}
\label{lem:vdiff}
For any $\pi$ and any $V$, 
for all states $s_\ell$, and for all $k\geq0$ it holds that:
\\
$\big((\calB^\pi)^k V\big)(s_\ell)-\tilde V^\pi(s_\ell)
=
\gamma^k
\expected_{a_{\ell}s_{\ell+1}...s_{\ell+k}|s_\ell}\big[
V(s_{\ell+k})-\tilde V^\pi(s_{\ell+k})
\big]$
.
\end{lemma}

\begin{proof}
Consider an arbitrary state $s_\ell$.
We use an induction on $k$.
For the base case, consider $k=0$ and observe that the claim follows trivially,
since
$\big((\calB^\pi)^0 V\big)(s_\ell)-
\big((\calB^\pi)^0 \tilde V^\pi\big)(s_\ell)
= V(s_\ell)-\tilde V^\pi(s_\ell)$.
For the induction hypothesis, assume the result holds for $k$.
Then consider:
\begin{align}
&
\big((\calB^\pi)^{k+1} V\big)(s_\ell)-
\big(\tilde V^\pi\big)(s_\ell)
\nonumber
\\
&=
\big((\calB^\pi)^{k+1} V\big)(s_\ell)-
\big((\calB^\pi)^{k+1} \tilde V^\pi\big)(s_\ell)
\quad\mbox{ (by Lemma~\ref{lem:a2}) }
\\
&=
\big(\calB^\pi(\calB^\pi)^k V\big)(s_\ell)-
\big(\calB^\pi(\calB^\pi)^k \tilde V^\pi\big)(s_\ell)
\\
&=
\expected_{a_\ell s_{\ell+1}|s_\ell}\Big[
r(s_\ell,a_\ell)-\tau\log\pi(a_\ell|s_\ell) + 
\gamma (\calB^\pi)^k V(s_{\ell+1})
\Big]
\nonumber
\\
&\hspace*{4mm}-
\expected_{a_\ell s_{\ell+1}|s_\ell}\Big[
r(s_\ell,a_\ell)-\tau\log\pi(a_\ell|s_\ell) + 
\gamma (\calB^\pi)^k\tilde V^\pi(s_{\ell+1})
\Big]
\\
&=
\gamma\expected_{a_\ell s_{\ell+1}|s_\ell}
\Big[(\calB^\pi)^k V(s_{\ell+1})-(\calB^\pi)^k\tilde V^\pi(s_{\ell+1})\Big]
\\
&=
\gamma\expected_{a_\ell s_{\ell+1}|s_\ell}
\Big[(\calB^\pi)^k V(s_{\ell+1})-\tilde V^\pi(s_{\ell+1})\Big]
\quad\mbox{ (by Lemma~\ref{lem:a2}) }
\\
&=
\gamma\expected_{a_\ell s_{\ell+1}|s_\ell}
\Big[
\gamma^k
\expected_{a_{\ell+1}s_{\ell+2}...s_{\ell+k+1}|s_{\ell+1}}
\big[V(s_{\ell+k+1})-\tilde V^\pi(s_{\ell+k+1})\big]
\Big]
\quad\mbox{ (by IH) }
\\
&=
\gamma^{k+1}
\expected_{a_{\ell}s_{\ell+1}...s_{\ell+k+1}|s_{\ell}}
\Big[
V(s_{\ell+k+1})-\tilde V^\pi(s_{\ell+k+1})
\Big]
,
\end{align}
establishing the claim.
\end{proof}

\begin{lemma}
\label{lem:vcontract}
For any $\pi$ and any $V$,
we have
$\big\|(\calB^\pi)^k V - \tilde V^\pi\big\|_\infty
\leq
\gamma^k\big\|V - \tilde V^\pi\big\|_\infty$.
\end{lemma}

\begin{proof}
Let $p^{(k)}(s_{\ell+k}|s_\ell)$ denote the conditional distribution over
the $k$th state, $s_{\ell+k}$, visited in a random walk starting from $s_\ell$,
which is induced by the environment and the policy $\pi$.
Consider
\begin{align}
\big\|(\calB^\pi)^k V - \tilde V^\pi\big\|_\infty
&=
\gamma^k
\max_{s_\ell}
\Big|
\expected_{a_{\ell}s_{\ell+1}...s_{\ell+k}|s_\ell}\big[
V(s_{\ell+k})-\tilde V^\pi(s_{\ell+k})
\big]\Big|
\quad\mbox{ (by Lemma~\ref{lem:vdiff}) }
\\
&=
\gamma^k
\max_{s_\ell}\Big|
\sum_{s_{\ell+k}}
p^{(k)}(s_{\ell+k}|s_\ell)
\big(V(s_{\ell+k})-\tilde V^\pi(s_{\ell+k})\big)
\Big|
\\
&=
\gamma^k
\max_{s_\ell}\Big|
p^{(k)}(:\!|s_\ell)
\cdot
\big(V-\tilde V^\pi\big)
\Big|
\\
&\leq
\gamma^k
\max_{s_\ell}
\|p^{(k)}(:\!|s_\ell)\|_1
\;
\|V-\tilde V^\pi\|_\infty
\quad\mbox{ (by H\"{o}lder's inequality) }
\\
&=
\gamma^k
\|V-\tilde V^\pi\|_\infty
.
\end{align}
\end{proof}

\begin{corollary}
\label{cor:vlb}
For any bounded $V$ and any $\epsilon>0$ there exists a $k_0$ such that
$(\calB^\pi)^kV\geq\tilde V^\pi-\epsilon$
for all $k\geq k_0$.
\end{corollary}

\begin{proof}
By Lemma~\ref{lem:vcontract} we have
$(\calB^\pi)^k V\geq \tilde V^\pi
-\gamma^k\big\|V - \tilde V^\pi\big\|_\infty$
for all $k\geq0$.
Therefore, for any $\epsilon>0$ there exists a $k_0$ such that
$\gamma^k\big\|V - \tilde V^\pi\big\|_\infty<\epsilon$
for all $k\geq k_0$, 
since $V$ is assumed bounded.
\end{proof}

Thus, any value function will converge to $\tilde V^\pi$
via repeated application of on-policy backups $\calB^\pi$.
Below we will also need to make use of the following monotonicity property
of the on-policy Bellman operator.

\begin{lemma}
\label{lem:Bpi monotonic}
For any $\pi$,
if $V^{(1)}\geq V^{(2)}$
then
$\calB^\pi V^{(1)}\geq\calB^\pi V^{(2)}$.
\end{lemma}

\begin{proof}
Assume $V^{(1)}\geq V^{(2)}$ and note that for any state $s_\ell$
\begin{align}
(\calB^\pi V^{(2)})(s_\ell)-(\calB^\pi V^{(1)})(s_\ell)
&=
\gamma\expected_{a_{\ell}s_{\ell+1}|s_\ell}\big[
V^{(2)}(s_{\ell+1})-V^{(1)}(s_{\ell+1})\big]
\\
&\leq0\quad\mbox{ since it was assumed that $V^{(2)}\leq V^{(1)}$.}
\end{align}
\end{proof}

\subsection{Proof of main optimality claims for \emph{off-policy} softmax updates}

Define the optimal value function by
\begin{align}
V^\dag(s) 
&= \max_\pi \ento(s,\pi)
= \max_\pi \tilde V^\pi(s) 
\mbox{ for all } s
.
\label{eq:vopt}
\end{align}

For $\tau>0$, 
define the softmax Bellman operator $\calB^*$ by
\begin{align}
(\calB^* V)(s) 
&= \tau\log\sum_a \exp\Big(\big(r(s,a)+\gamma\expected_{s'|s,a}[V(s')]\big)/\tau\Big)
\\
&= F_\tau(Q(s,:))
\quad\mbox{ where }\quad
Q(s,a) = r(s,a) + \gamma\expected_{s'|s,a}[V(s')]
\quad\mbox{ for all } a
.
\label{eq:bellman F}
\end{align}

\begin{lemma}
\label{lem:fixed point}
For $\gamma<1$,
the fixed point of the softmax Bellman operator, $\Vstar=\calB^* \Vstar$,
exists and is unique.
\end{lemma}

\begin{proof}
First observe that the softmax Bellman operator is a contraction in the
infinity norm.
That is, consider two value functions, $V^{(1)}$ and $V^{(2)}$, and
let $p(s'|s,a)$ denote the state transition probability function
determined by the environment.
We then have
\begin{align}
\left\|\calB^* V^{(1)}-\calB^* V^{(2)}\right\|_\infty
&=
\max_s
\left|(\calB^* V^{(1)})(s)-(\calB^* V^{(2)})(s)\right|
\\
&=
\max_s
\left|F_\tau\big(Q^{(1)}(s,:)\big)-F_\tau\big(Q^{(2)}(s,:)\big)\right|
\\
&\leq
\max_s\max_a\left|Q^{(1)}(s,a)-Q^{(2)}(s,a)\right|
\quad\mbox{ (by Corollary~\ref{cor:qcontract}) }
\\
&=
\gamma\max_s\max_a
\left|\expected_{s'|s,a}\big[V^{(1)}(s')-V^{(2)}(s')\big]\right|
\\
&=
\gamma\max_s\max_a
\left|p(:|s,a)\cdot\big(V^{(1)}-V^{(2)}\big)\right|
\\
&\leq
\gamma\max_s\max_a
\|p(:|s,a)\|_1\;\|V^{(1)}-V^{(2)}\|_\infty
\quad\mbox{(H\"{o}lder's inequality) }
\\
&=
\gamma
\|V^{(1)}-V^{(2)}\|_\infty
<
\|V^{(1)}-V^{(2)}\|_\infty
\mbox{ if } \gamma<1
.
\end{align}
The existence and uniqueness of $\Vstar$
then follows from the contraction map fixed-point theorem
\cite{bertsekas95}.
\end{proof}

\begin{lemma}
\label{lem:dom}
For any $\pi$, 
if $V\geq\calB^*V$ then $V\geq(\calB^\pi)^k V$ for all $k$.
\end{lemma}

\begin{proof}
Observe for any $s$ that the assumption implies
\begin{align}
V(s) &\geq (\calB^*V)(s)
\\
&= \max_{\tilde\pi(:|s)\in\Delta}
\sum_a\tilde\pi(a|s)
\Big(r(s,a)
+\gamma\expected_{s'|s,a}[V(s')]
-\tau\log\tilde\pi(a|s)\Big)
\\
&\geq
\sum_a\pi(a|s)
\Big(r(s,a)
+\gamma\expected_{s'|s,a}[V(s')]
-\tau\log\pi(a|s)\Big)
\\
&=(\calB^\pi V)(s)
.
\end{align}
The result then follows by the monotonicity of $\calB^\pi$ 
(Lemma~\ref{lem:Bpi monotonic}).
\end{proof}

\begin{corollary}
\label{cor:vopt}
For any $\pi$,
if $V$ is bounded and $V\geq\calB^*V$, then $V\geq\tilde V^\pi$.
\end{corollary}

\begin{proof}
Consider an arbitrary policy $\pi$.
If $V\geq\calB^*V$, then by Corollary~\ref{cor:vopt} we have
$V\geq(\calB^\pi)^k V$ for all $k$.
Then by Corollary~\ref{cor:vlb},
for any $\epsilon>0$ there exists a $k_0$ such that
$V\geq(\calB^\pi)^kV\geq\tilde V^\pi-\epsilon$
for all $k\geq k_0$
since $V$ is bounded;
hence $V\geq\tilde V^\pi-\epsilon$ for all $\epsilon>0$.
We conclude that $V\geq\tilde V^\pi$.
\end{proof}

Next, given the existence of $\Vstar$, we define a specific policy $\pistar$
as follows
\begin{align}
\pistar(:\!|s) &= \vec{f}_\tau\big(Q^*(s,:)\big),
\quad\mbox{ where }
\label{eq:pi*}
\\
Q^*(s,a) &=r(s,a)+\gamma\expected_{s'|s,a}[\Vstar(s')].
\label{eq:q*}
\end{align}
Note that we are simply defining $\pistar$ at this stage
and have not as yet proved it has any particular properties;
but we will see shortly that it is, in fact, an optimal policy.

\begin{lemma}
$\Vstar=\tilde V^{\pistar}$;
that is, for $\pistar$ 
defined in \eqref{eq:pi*},
$\Vstar$ gives its entropy regularized expected return from any state.
\label{lem:vpi*=v*}
\end{lemma}

\begin{proof}
We establish the claim by showing
$\calB^*\tilde V^{\pistar}=\tilde V^{\pistar}$.
In particular, for an arbitrary state $s$ consider
\begin{align}
(\calB^*\tilde V^{\pistar})(s)
&=
F_\tau\big(\tilde Q^{\pistar}(s,:)\big)
&\mbox{ by \eqref{eq:bellman F}}
\\
&=\pistar(:\!|s)\cdot\big(\tilde Q^{\pistar}(s,:)-\tau\log\pistar(:\!|s)\big)
&\mbox{ by Lemma~\ref{lem:a1}}
\\
&= \tilde V^{\pistar}(s)
&\mbox{ by Lemma~\ref{lem:a2}}.
\end{align}
\end{proof}

\begin{theorem}
The fixed point of the softmax Bellman operator is the optimal value function:
$\Vstar=V^\dag$.
\label{thm:bellman=opt}
\end{theorem}

\begin{proof}
Since $\Vstar\geq\calB^*\Vstar$ (in fact, $\Vstar=\calB^*\Vstar$)
we have $\Vstar\geq\tilde V^\pi$ for all $\pi$ 
by Corollary~\ref{cor:vopt},
hence $\Vstar\geq V^\dag$.
Next observe that by Lemma~\ref{lem:vpi*=v*}
we have
$V^\dag\geq\tilde V^{\pi^*}=\Vstar$.
Finally,
by Lemma~\ref{lem:fixed point},
we know that the fixed point 
$\Vstar=\calB^* \Vstar$ is unique,
hence $V^\dag=\Vstar$.
\end{proof}

\begin{corollary}[Optimality Implies Consistency]
The optimal state value function $\Vstar$ and optimal policy $\pistar$ satisfy
$\Vstar(s)=r(s,a)+\gamma\expected_{s'|s,a}[\Vstar(s')]-\tau\log\pistar(a|s)$
for every state $s$ and action $a$.
\label{cor:opt=>consistent}
\end{corollary}

\begin{proof}
First note that
\begin{align}
Q^*(s,a) &= r(s,a) + \gamma\expected_{s'|s,a}[\Vstar(s')]
&\mbox{ by \eqref{eq:q*}}
\\
&= r(s,a) + \gamma\expected_{s'|s,a}[V^{\pistar}(s')]
&\mbox{ by Lemma~\ref{lem:vpi*=v*}}
\\
&=Q^{\pistar}(s,a)
&\mbox{ by \eqref{eq:Qvdef}}
.
\end{align}
Then observe that for any state $s$,
\begin{align}
\Vstar(s) 
&= F_\tau\big(Q^*(s,:)\big)
&\mbox{ by \eqref{eq:bellman F}}
\\
&= F_\tau\big(Q^{\pistar}(s,:)\big)
&\mbox{ from above }
\\
&=
\pistar(:\!|s)\cdot\big(Q^{\pistar}(s,:)-\tau\log\pistar(:\!|s)\big)
&\mbox{ by Lemma~\ref{lem:a1}}
\\
&=
Q^{\pistar}(s,a)-\tau\log\pistar(a|s)
\mbox{ for all } a
&\mbox{ by Corollary~\ref{cor:a1}}
\\
&=
Q^*(s,a)-\tau\log\pistar(a|s)
\mbox{ for all } a
&\mbox{ from above}
.
\end{align}
\end{proof}

\begin{corollary}[Consistency Implies Optimality]
If $V$ and $\pi$ satisfy,
for all $s$ and $a$:
\\
$V(s)=r(s,a)+\gamma\expected_{s'|s,a}[V(s')]-\tau\log\pi(a|s)$;
then $V=\Vstar$ and $\pi=\pistar$.
\label{cor:consistent=>opt}
\end{corollary}

\begin{proof}
We will show that satisfying the constraint for every $s$ and $a$ implies 
${\cal B}^*V=V$; 
it will then immediately follow 
that $V=\Vstar$ and $\pi=\pistar$
by Lemma~\ref{lem:fixed point}. 
Let $Q(s,a)=r(s,a)+\gamma\expected_{s'|s,a}[V(s')]$.
Consider an arbitrary state $s$, and observe that
\begin{align}
({\cal B}^*V)(s) 
&= F_\tau\big(Q(s,:)\big)
\quad\mbox{ (by \eqref{eq:bellman F})}
\\
&=
\max_{\vec{\pi}\in\Delta}\Bigl\{
\vec{\pi}\cdot\big(Q(s,:)-\tau\log\vec{\pi}\big)\Bigl\}
\quad\mbox{ (by Lemma~\ref{lem:a1})}
\\
&=
Q(s,a)-\tau\log\pi(a|s)
\mbox{ for all } a
\quad\mbox{ (by Corollary~\ref{cor:a2})}
\\
&=
r(s,a)+\gamma\expected_{s'|s,a}[V(s')]-\tau\log\pi(a|s)
\mbox{ for all } a
\quad\mbox{ (by definition of $Q$ above)}
\\
&=
V(s)
\quad\mbox{ (by the consistency assumption on $V$ and $\pi$)}
.
\end{align}
\end{proof}

\subsection{Proof of Theorem~\ref{thm:1} from Main Text}

{\bf Note}: 
Theorem~\ref{thm:1} from the main body was stated under an assumption
of deterministic dynamics. 
We used this assumption in the main body merely to keep presentation simple
and understandable.  
The development given in this appendix considers the more general
case of a stochastic environment.
We give the proof here for the more general setting;
the result stated in Theorem~\ref{thm:1} follows as a special case.

\begin{proof}
Assuming a stochastic environment, as developed in this appendix,
we will establish that the optimal policy and state value function,
$\pistar$ and $\Vstar$ respectively, satisfy
\begin{align}
\Vstar(s) &= -\tau\log\pistar(a|s)+r(s,a)+\gamma\expected_{s'|s,a}[\Vstar(s')]
\label{eq:v pi consistent general}
\end{align}
for all $s$ and $a$.
Theorem~\ref{thm:1} will then follow as a special case.

Consider the policy $\pistar$ defined in \eqref{eq:pi*}.
From Corollary~\ref{lem:vpi*=v*} we know that $\tilde V^{\pistar}=\Vstar$
and from Theorem~\ref{thm:bellman=opt} we know $\Vstar=V^\dag$,
hence $\tilde V^{\pistar}=V^\dag$;
that is, 
$\pistar$ is the optimizer of $\ento(s,\pi)$
for any state $s$ (including $s_0$).
Therefore, this must be the same $\pistar$ 
as considered in the premise.
The assertion \eqref{eq:v pi consistent general}
then follows directly from Corollary~\ref{cor:opt=>consistent}.
\end{proof}


\subsection{Proof of Corollary~\ref{cor:2} from Main Text}

{\bf Note}: 
We consider the more general case of a stochastic environment as
developed in this appendix.
First note that the consistency property 
for the stochastic case
\eqref{eq:v pi consistent general}
can be rewritten as
\begin{align}
\expected_{s'|s,a}\big[
-\Vstar(s)+\gamma\Vstar(s') +r(s,a) -\tau\log\pistar(a|s)
\big]
&= 0
\label{eq:cons1}
\end{align}
for all $s$ and $a$.
For a stochastic environment,
the generalized version of \eqref{eq:pathwise}
in Corollary~\ref{cor:2} can then be 
expressed as
\begin{align}
\expected_{s_2...s_t|s_1,a_1...a_{t-1}}\bigg[
-\Vstar(s_1)+\gamma^{t-1}\Vstar(s_t)
+\sum_{i=1}^{t-1}\gamma^{i-1}\big(
r(s_i,a_i)-\tau\log\pistar(a_i|s_i)
\big)
\bigg]
&=0
\label{eq:claim cor2}
\end{align}
for all states $s_1$ and action sequences $a_1...a_{t-1}$.
We now show that \eqref{eq:cons1} implies \eqref{eq:claim cor2}.

\begin{proof}
Observe that by \eqref{eq:cons1} we have
\begin{align}
0 &=
\expected_{s_2...s_t|s_1,a_1...a_{t-1}}\bigg[
\sum_{i=1}^{t-1}\gamma^{i-1}\Big(
-\Vstar(s_i)+\gamma\Vstar(s_{i+1})
+
r(s_i,a_i)-\tau\log\pistar(a_i|s_i)
\Big)
\bigg]
\\
&=
\expected_{s_2...s_t|s_1,a_1...a_{t-1}}\bigg[
\sum_{i=1}^{t-1}\gamma^{i-1}\Big(
-\Vstar(s_i)+\gamma\Vstar(s_{i+1})
\Big)
\nonumber
\\
&\hspace*{68mm}
+
\sum_{i=1}^{t-1}\gamma^{i-1}\Big(
r(s_i,a_i)-\tau\log\pistar(a_i|s_i)
\Big)
\bigg]
\\
&=
\expected_{s_2...s_t|s_1,a_1...a_{t-1}}\bigg[
-\Vstar(s_1)+\gamma^{t-1}\Vstar(s_t)
+
\sum_{i=1}^{t-1}\gamma^{i-1}\Big(
r(s_i,a_i)-\tau\log\pistar(a_i|s_i)
\Big)
\bigg]
\end{align}
by a telescopic sum on the first term, 
which yields the result.
\end{proof}

\subsection{Proof of Theorem~\ref{thm:2} from Main Text}

{\bf Note}:
Again, we consider the more general case of a stochastic environment.
The consistency property in this setting is given by
\eqref{eq:v pi consistent general} above.

\begin{proof}
Consider a policy $\pi_\theta$ and value function $V_\phi$ that
satisfy the general consistency property for a stochastic environment:
$V_\phi(s) = -\tau\log\pi_\theta(a|s)+r(s,a)
+\gamma\expected_{s'|s,a}[V_\phi(s')]$
for all $s$ and $a$.
Then by Corollary~\ref{cor:consistent=>opt},
we must have
$V_\phi=\Vstar$ and $\pi_\theta=\pistar$.
Theorem~\ref{thm:2} follows as a special case when the 
environment is deterministic.
\end{proof}

\end{document}